\documentclass[10pt,journal,compsoc]{IEEEtran}

\usepackage{algorithmic}
\usepackage{amsmath}
\usepackage{xspace}
\usepackage{array}
\usepackage{url}
\usepackage{dirtytalk}
\usepackage{amssymb}
\usepackage{amsthm}
\usepackage{dirtytalk}
\usepackage{multirow}
\usepackage{booktabs}
\usepackage{lipsum}  
\usepackage{nicefrac}
\usepackage[colorlinks=true,linkcolor=red]{hyperref}

\newtheorem{remark}{Remark}

\newtheorem{theorem}{Theorem}
\newtheorem{theorem_supp}{Theorem}

\newtheorem{lemma}{Lemma}
\newtheorem{lemma_supp}{Lemma}

\newtheorem{proposition_supp}{Proposition}

\ifCLASSOPTIONcompsoc
  \usepackage[nocompress]{cite}
\else
  \usepackage{cite}
\fi
\ifCLASSINFOpdf
  \usepackage[pdftex]{graphicx}
  \graphicspath{{Figures/}}
  \DeclareGraphicsExtensions{.pdf,.jpeg,.png,.jpg}
\else
  \usepackage[dvips]{graphicx}
  \graphicspath{{Figures/}}
  \DeclareGraphicsExtensions{.eps}
\fi
\ifCLASSOPTIONcompsoc
  \usepackage[caption=false,font=footnotesize,labelfont=sf,textfont=sf]{subfig}
\else
  \usepackage[caption=false,font=footnotesize]{subfig}
\fi

\newcommand{\ie}{\textit{i.e.}\xspace}
\newcommand{\eg}{\textit{e.g.}\xspace}
\newcommand{\etal}{\textit{et al.}\xspace}

\usepackage{hyperref}

\begin{document}
%
% paper title
% Titles are generally capitalized except for words such as a, an, and, as,
% at, but, by, for, in, nor, of, on, or, the, to and up, which are usually
% not capitalized unless they are the first or last word of the title.
% Linebreaks \\ can be used within to get better formatting as desired.
% Do not put math or special symbols in the title.
\title{Network Moments: Extensions and Sparse-Smooth Attacks}
%
%
% author names and IEEE memberships
% note positions of commas and nonbreaking spaces ( ~ ) LaTeX will not break
% a structure at a ~ so this keeps an author's name from being broken across
% two lines.
% use \thanks{} to gain access to the first footnote area
% a separate \thanks must be used for each paragraph as LaTeX2e's \thanks
% was not built to handle multiple paragraphs
%
%
%\IEEEcompsocitemizethanks is a special \thanks that produces the bulleted
% lists the Computer Society journals use for "first footnote" author
% affiliations. Use \IEEEcompsocthanksitem which works much like \item
% for each affiliation group. When not in compsoc mode,
% \IEEEcompsocitemizethanks becomes like \thanks and
% \IEEEcompsocthanksitem becomes a line break with idention. This
% facilitates dual compilation, although admittedly the differences in the
% desired content of \author between the different types of papers makes a
% one-size-fits-all approach a daunting prospect. For instance, compsoc 
% journal papers have the author affiliations above the "Manuscript
% received ..."  text while in non-compsoc journals this is reversed. Sigh.

\author{
  \IEEEauthorblockN{
    Modar Alfadly\IEEEauthorrefmark{1}\textsuperscript{\textsection},
    Adel Bibi\IEEEauthorrefmark{1}\textsuperscript{\textsection},
    Emilio Botero\IEEEauthorrefmark{1}\IEEEauthorrefmark{3},
    Salman Alsubaihi\IEEEauthorrefmark{1} and
    Bernard Ghanem\IEEEauthorrefmark{1}
  }
  \\
  \IEEEauthorblockA{\IEEEauthorrefmark{1} King Abdullah University of Science and Technology (KAUST), Thuwal, Saudi Arabia}
  \\
  \IEEEauthorblockA{\IEEEauthorrefmark{3} Universit\'e de Montr\'eal, Quebec, Canada}
  \\
  \{modar.alfadly,adel.bibi,salman.subaihi,bernard.ghanem\}@kaust.edu.sa, emilio.botero@umontreal.ca
}

% for Computer Society papers, we must declare the abstract and index terms
% PRIOR to the title within the \IEEEtitleabstractindextext IEEEtran
% command as these need to go into the title area created by \maketitle.
% As a general rule, do not put math, special symbols or citations
% in the abstract or keywords.
\IEEEtitleabstractindextext{
\begin{abstract}
    The impressive performance of deep neural networks (DNNs) has immensely strengthened the line of research that aims at theoretically analyzing their effectiveness. This has incited research on the reaction of DNNs to noisy input, namely developing adversarial input attacks and strategies that lead to robust DNNs to these attacks. To that end, in this paper, we derive exact analytic expressions for the first and second moments (mean and variance) of a small piecewise linear (PL) network (Affine, ReLU, Affine) subject to Gaussian input. In particular, we generalize the second-moment expression of \cite{bibi2018analytic} to arbitrary input Gaussian distributions, dropping the zero-mean assumption. We show that the new variance expression can be efficiently approximated leading to much tighter variance estimates as compared to the preliminary results of Bibi \etal \cite{bibi2018analytic}. Moreover, we experimentally show that these expressions are tight under simple linearizations of deeper PL-DNNs, where we investigate the effect of the linearization sensitivity on the accuracy of the moment estimates. Lastly, we show that the derived expressions can be used to construct sparse and smooth Gaussian adversarial attacks (targeted and non-targeted) that tend to lead to perceptually feasible input attacks.
\end{abstract}

% Note that keywords are not normally used for peerreview papers.
}

% make the title area
\maketitle
\begingroup\renewcommand\thefootnote{\textsection}
\begin{NoHyper}
\footnotetext{Equal contribution}
\end{NoHyper}
% To allow for easy dual compilation without having to reenter the
% abstract/keywords data, the \IEEEtitleabstractindextext text will
% not be used in maketitle, but will appear (i.e., to be "transported")
% here as \IEEEdisplaynontitleabstractindextext when the compsoc 
% or transmag modes are not selected <OR> if conference mode is selected 
% - because all conference papers position the abstract like regular
% papers do.
\IEEEdisplaynontitleabstractindextext
% \IEEEdisplaynontitleabstractindextext has no effect when using
% compsoc or transmag under a non-conference mode.

% For peer review papers, you can put extra information on the cover
% page as needed:
% \ifCLASSOPTIONpeerreview
% \begin{center} \bfseries EDICS Category: 3-BBND \end{center}
% \fi
%
% For peerreview papers, this IEEEtran command inserts a page break and
% creates the second title. It will be ignored for other modes.
\IEEEpeerreviewmaketitle

% Computer Society journal (but not conference!) papers do something unusual
% with the very first section heading (almost always called "Introduction").
% They place it ABOVE the main text! IEEEtran.cls does not automatically do
% this for you, but you can achieve this effect with the provided
% \IEEEraisesectionheading{} command. Note the need to keep any \label that
% is to refer to the section immediately after \section in the above as
% \IEEEraisesectionheading puts \section within a raised box.
\IEEEraisesectionheading{\section{Introduction}\label{sec:introduction}}
\IEEEPARstart{D}{eep} neural networks (DNNs) have revolutionized not only the computer vision and machine learning communities but several other fields throughout science and engineering such as natural language processing, bioinformatics and medicine \cite{lecun2015yoshua}. While major advances in the areas of object classification \cite{krizhevsky2012imagenet}, and speech recognition \cite{hinton2012deep} to name a few, have been attributed to DNNs, a rigorous theoretical understanding of their effectiveness remains elusive. For instance, while DNNs have shown impressive performance on visual recognition tasks, they still exhibit uncouth behaviour when they are subject to carefully tailored inputs \cite{szegedy2013intriguing}. Many prior works show that it is rather easy, through simple routines, to craft imperceptible input perturbations, referred to as adversarial attacks. Such attacks can result in a drastic negative effect on the classification performance of many popular deep models \cite{goodfellow2014explaining,moosavi2016deepfool,szegedy2013intriguing}. Even more surprisingly, one can design such adversarial perturbations to be agnostic to both the input image and the network architecture \cite{moosavi2016universal}, which are referred to as universal perturbations. Unfortunately, less progress has been made towards systematically addressing and understanding this challenge. One of the early and naive approaches towards addressing this nuisance is simply through augmenting the training dataset with data corrupted with adversaries. While this has been shown to improve network robustness against such adversaries \cite{goodfellow2014explaining,moosavi2016deepfool}, unfortunately, this is a vacuous brute force approach that does not provide insights on the reasons behind such behaviour. Moreover, it does not scale for large dimensional inputs, as the amount of corresponding augmentation has to necessarily be prohibitively large to capture the variation in input space. This effectively deems the augmentation approach infeasible in large dimensions.

\begin{figure}[!t]
\centering
    \includegraphics[width = 0.48\textwidth]{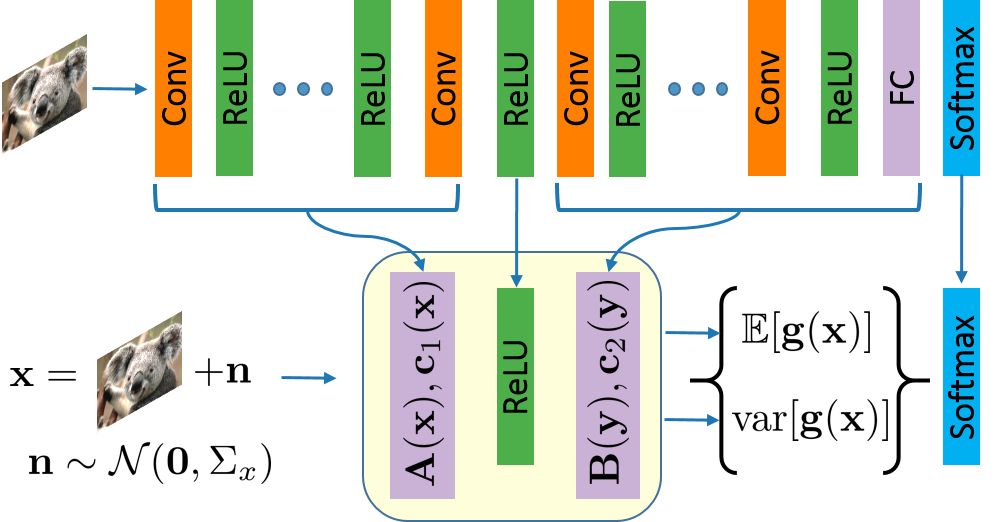}
	\caption{\textbf{Two-stage linearization of an arbitrarily deep network.} Any PL-DNN can be linearized before and after a given ReLU through a two-stage linearization truncating it into a (Affine, ReLU, Affine) network, whose  $1^{\text{st}}$ and $2^{\text{nd}}$ moments can be derived  analytically when it is exposed to Gaussian input noise. We show that these moments are helpful in predicting how PL-DNNs react to noise and in constructing adversarial Gaussian input attacks.}
    \label{pull_fig}
\end{figure}

In this paper, we derive expressions for the first and second moments (the mean and consequently the variance), referred to as Network Moments, of a small piecewise linear (PL) network in the form of (Affine, ReLU, Affine) subject to a general Gaussian input. The preliminary version of these Network Moments were derived and analyzed in \cite{bibi2018analytic}. Beyond these preliminary results, we derive in this paper a new variance expression, which does not claim any assumptions on the mean or the covariance of the input Gaussian. This generalizes the previous result in \cite{bibi2018analytic}, which only holds under a zero mean input assumption. These expressions provide a powerful tool for analyzing deeper PL-DNNs by means of two-stage linearization (as shown in Figure \ref{pull_fig}) with a plethora of applications. For instance, it has been shown that such expressions can be quite useful in training robust networks very efficiently \cite{alfadly2019train}, avoiding any need for noisy data augmentation. In particular, empirical evidence in \cite{alfadly2019train} indicates that simple regularizers based on the mean and variance expressions can boost network robustness by two orders of magnitude not only against Gaussian attacks but also against other popular adversarial attacks (\eg PGD, LBFGS \cite{szegedy2013intriguing}, FGSM \cite{goodfellow2014explaining} and DF2 \cite{moosavi2016deepfool}). In this paper, we show that network moments can be used to systematically design Gaussian distributions that can serve as input adversaries. In particular, we conduct several experiments on MNIST \cite{lecun1998mnist} and Facial Emotion Recognition datasets \cite{goodfellow2015} to demonstrate that these expressions can be used to craft sparse and smooth Gaussian attacks that are structured and perceptually feasible, \ie they exhibit interesting semantic information aligned with human perception.

\vspace{3pt} \noindent\textbf{Contributions.} \textbf{(i)} We provide a fresh perspective on analyzing PL-DNNs by deriving closed form expressions for the output mean and variance of a network in the form (Affine, ReLU, Affine) in the presence of general Gaussian input noise. In particular, we generalize the results of \cite{bibi2018analytic} and derive a closed form expression for the second moment under no assumptions on the mean nor covariance of the input Gaussian. Through network linearization, extensive experiments show that the new expression for the output variance can be efficiently approximated leading to much tighter second-moment estimates than that of \cite{bibi2018analytic}. \textbf{(ii)} We formalize a new objective as a function of the derived output mean and variance to construct sparse and smooth Gaussian adversarial attacks. We conduct extensive experiments on both MNIST and Facial Emotion datasets demonstrating that the constructed adversaries are perceptually feasible.

\section{Related Work}

Despite the impressive performance of deep neural networks on visual recognition tasks, their performance can still be drastically obstructed in the presence of small imperceptible adversarial noise \cite{goodfellow2014explaining,moosavi2016deepfool,szegedy2013intriguing}. Alarmingly, such adversaries are abundant and easy to construct, where in some scenarios constructing an adversary is as simple as performing a single gradient ascent step of some loss function with respect to the input \cite{szegedy2013intriguing}. More surprisingly, there exist deterministic input samples that are agnostic of both the input and network architecture that can cause severe reduction in the network performance \cite{moosavi2016universal}. Moreover, in some extreme cases, it can be sufficient to perturb a single input pixel that can result in a misclassification rate as high as $70\%$ on popular benchmarks \cite{su2017one}.

This nuisance is serious and menacing and has to be addressed, particularly since DNNs are now deployed in sensitive real-world applications (\eg self driving cars). Thereafter, there have been several directions towards understanding and circumventing this. Early works aimed at analyzing the behaviour of DNNs in the general presence of input noise. For instance, Fawzi \etal \cite{fawzi2016measuring} proposed a generic probabilistic framework for analyzing the robustness of a classifier under different nuisance factors. Another seminal work particularly assessed  the robustness of a classifier undergoing geometric transformations \cite{fawzi2015manitest}. On the other hand, there has been several other works on the design and training of networks that are robust against adversarial attacks. One of the earliest approaches on this was the direct augmentation of adversarial samples to the training data, which has been shown to indeed lead to more robust networks \cite{goodfellow2014explaining,moosavi2016deepfool}. Later, the work of \cite{madry2017towards} adopted a similar strategy but by incorporating the adversarial augmentation during the iterative training process. In particular, it was shown that one can achieve significant boosts in network robustness against first-order adversarial attacks, \ie attacks that depend only on gradient information, by minimizing the worst adversarial loss over all bounded energy (often measured in $\ell_\infty$ norm) perturbations around a given input.

Since then, there has been a surge in literature studying verification approaches for DNNs. In this line of work, the aim is to design networks that are accurate and provably robust against all bounded input attacks. In general, verification approaches can be coarsely categorized as exact or relaxed verifiers. The former try to find the exact largest adversarial loss over all possible bounded inputs. Such verifiers often require piecewise linear networks and rely on either Mixed Integer Solvers (MIS) \cite{cheng2017maximum,lomuscio2017approach} or on Satisfiability Modulo Theories (SMT) solvers \cite{scheibler2015towards,katz2017reluplex}. These verifiers are too expensive for DNNs due to their NP-complete nature. Relaxed verifiers on the other hand scale better, since they only find an upper bound to the worst adversarial loss \cite{zhang2018efficient,wong2017provable}. There has been several new directions that aim at addressing the verification problem by constructing networks with smoothed decision boundaries \cite{lecuyer2019certified,cohen_randomized_1}. 

In this paper, we are not concerned with such techniques but only focus on analyzing the behaviour of networks in the presence of input noise. We focus our analysis on PL-DNNs with ReLU activations. Unlike previous work, we study how the probabilistic moments of the output of a PL-DNN with a Gaussian input can be computed analytically. A similar work to ours is \cite{gast2018lightweight}, where the probabilistic output mean and variance of a deep network are estimated by propagating the estimates of the moments \emph{per layer} under the assumption that the joint distribution after each affine layer is still Gaussian (through the central limit theorem). On the contrary, we derive the \textit{exact} first and second moments of a simple two-layer (Affine, ReLU, Affine) network. We extrapolate these expressions to deeper PL-DNNs by employing a simple two-stage linearization step that locally approximates them with a (Affine, ReLU, Affine) network. Since these expressions are a function of the noise parameters, they are particularly useful in analyzing and inferring the behaviour of the original PL-DNN \emph{without} having to probe the network with inputs sampled from the noise distribution as regularly done in previous work \cite{goodfellow2014explaining,moosavi2016deepfool}.

\section{Network Moments\texorpdfstring{\protect\footnote{All proofs are omitted for the \textbf{Appendix}.}}{}}

We start by analyzing a particularly shaped network in the form of (Affine, ReLU, Affine) in the presence of Gaussian input noise. The functional form of the network of interest is given as $\mathbf{g}(\mathbf{x}) = \mathbf{B} \max\left(\mathbf{A} \mathbf{x} + \mathbf{c}_1,\mathbf{0}_p \right) + \mathbf{c}_2$, where $\max(.)$ is an  element wise operator. The affine mappings can be of any size, and we assume throughout the paper that $\mathbf{A} \in \mathbb{R}^{p \times n}$ and $\mathbf{B} \in \mathbb{R}^{d \times p}$, where $d$ is the number of output logits. Note that $\mathbf{A}$ and $\mathbf{B}$ can be of any structure (circular or Toeplitz) generalizing both fully connected and convolutional layers.

In this section, we analyze $\mathbf{g}$ when $\mathbf{x}$ is a Gaussian random vector, \ie $\mathbf{x} \sim \mathcal{N}(\mu_x,\Sigma_x)$. Seeking the probability density function (PDF) through the nonlinear random variable mapping $\mathbf{g}$ is possible for when $\mathbf{B} = \mathbf{I}$ but much more difficult for arbitrary $\mathbf{B}$ in general. Thus, we instead focus on deriving the \textit{probabilistic moments} of the unknown distribution of $\mathbf{g}(\mathbf{x})$. For ease of notation, we denote $\mathbf{g}_{i}(.)$ as the $i^{\text{th}}$ function in $\mathbf{g}(.)$, \ie $\mathbf{g}_i(\mathbf{x}) =$ $\mathbf{B}(i,:) \text{max}(\mathbf{A}\mathbf{x} + \mathbf{c}_1,\mathbf{0}_p) + \mathbf{c}_2(i)$. At first, and for completeness, we present the results of our preliminary work \cite{bibi2018analytic}, where the first moment (mean) expression is derived for a general Gaussian input distribution, while the second moment is derived under a zero input mean assumption, \ie $\mu_x = \mathbf{0}_n$ with $\mathbf{c}_1 = \mathbf{0}$. We then derive and generalize the expression for the second moment of $\mathbf{g}(\mathbf{x})$ for a generic Gaussian distribution under no assumptions in Lemma \ref{lemma4}.

\subsection{Deriving the \texorpdfstring{$1^{\text{st}}$}{first} Output Moment: \texorpdfstring{$\mathbb{E}[\mathbf{g}(\mathbf{x})]$}{E[g(x)]}}

\noindent To derive the first moment of $\mathbf{g}(\mathbf{x})$, we first consider the scalar function $q(x) = \max(x,0)$ acting on a single Gaussian random variable $x$.

\begin{remark}
\label{remark1}
The PDF of $q(x) = \max(x,0) : \mathbb{R} \rightarrow \mathbb{R}$ where $x \sim \mathcal{N}\left(\mu_x, \sigma_x^2 \right)$ is:
\begin{align}
f_q(x) = Q\left(\frac{\mu_x}{\sigma_x}\right) \delta(x) + f_x(x) u(x)
\notag
\end{align}
where $Q(.)$ is the Gaussian Q-function, $\delta(x)$ is the dirac function, $f_x(x)$ is the Gaussian PDF, and  $u(.)$ is the unit step function. It follows directly that $\mathbb{E}[q(x)] = \frac{\sigma_x}{\sqrt{2\pi}}$ when $\mu_x=0$.
\end{remark}

\noindent Now, we present the first moment of $\mathbf{g}(\mathbf{x})$.

\begin{theorem}
\label{theo1}
For any function in the form of $\mathbf{g}(\mathbf{x})$ where $\mathbf{x} \sim \mathcal{N} \left(\mu_x, \Sigma_x \right)$, we have:
\begin{equation}
\begin{aligned}
\mathbb{E}[\mathbf{g}_i(\mathbf{x})] =  \sum_{v=1}^p &\mathbf{B}(i,v)\left(  \frac{1}{2}  \bar{\mu}_{v}  -  \frac{1}{2}  \bar{\mu}_{v}  \text{erf}\left(\frac{-\bar{\mu}_{v}}{\sqrt{2} \bar{\sigma}_{v}}\right) \right. \\
& \left. +  \frac{1}{\sqrt{2 \pi}}\bar{\sigma}_{v} \exp\left(-\frac{\bar{\mu}_{v}^2}{2 \bar{\sigma}_{v}^2} \right) \right) + \mathbf{c}_2(i) 
\notag
\end{aligned}
\end{equation}
where $\bar{\mu}_{v} = \left(\mathbf{A} \mu_x + \mathbf{c}_1\right)(v)$,  $\bar{\Sigma} = \mathbf{A} \Sigma_x \mathbf{A}^\top$, $\bar{\sigma}_{v}^2 = \bar{\Sigma}(v,v)$ and $\text{erf}\left(x\right) = \frac{2}{\sqrt{\pi}} \int_0^x \text{e}^{-t^2}dt$ is the error function.
\end{theorem}

\subsection{Deriving the \texorpdfstring{$2^{\text{nd}}$}{second} Output Moment: \texorpdfstring{$\mathbb{E}[\mathbf{g}^2(\mathbf{x})]$}{E[g(x)g(x)]}}

\noindent Here, we need three pre-requisite lemmas: one that characterizes the PDF of a squared ReLU (Lemma \ref{lemma1}), another that extends Price's Theorem \cite{price1958useful} (Lemma \ref{price_exten}), and one that derives the first moment of the product of two ReLU functions (Lemma \ref{lemma3}). 

\begin{lemma}
\label{lemma1}
The PDF of $q^2(x) = \max^2(x,0) : \mathbb{R} \rightarrow \mathbb{R}$ where $x \sim \mathcal{N}\left(0,\sigma_x^2 \right)$ is :
\begin{equation}
\begin{aligned}
f_{q^2}(x) =  \frac{1}{2} \delta(x) + \frac{1}{2\sqrt{x}} f_x(\sqrt{x}) u(\sqrt{x})
\notag
\end{aligned}
\end{equation}
and its first moment is $\mathbb{E}[q^2(x)] = \frac{\sigma_x^2}{2}$.
\end{lemma}

\begin{lemma}
\label{price_exten}
Let $\mathbf{x} \in \mathbb{R}^{n} \sim \mathcal{N}(\mu_x,\Sigma_x)$ for any even p, where $\sigma_{ij} = ~~ \Sigma_x(i,j) ~\forall i\neq j$. Under mild assumptions on the nonlinear map $\Psi :
\mathbb{R}^{n} \rightarrow \mathbb{R}$, we have $\frac{\partial^{\frac{p}{2}} \mathbb{E}[\Psi(\mathbf{x})]}{\prod_{\forall \text{odd} i}\partial \sigma_{ii+1}}$ $= \mathbb{E} [\frac{\partial^p \Psi(\mathbf{x})}{\partial x_1 \dots \partial
x_p} ]$.
\end{lemma}

\noindent Lemma (\ref{price_exten}) relates the mean of the gradients/subgradients of any nonlinear function to the gradients/subgradients of the mean of that function. This lemma has Price's theorem \cite{price1958useful} as a special case when the function $\Psi(\mathbf{x})$ has the structure $\Psi(\mathbf{x}) = \prod_i^n \Psi_i(x_n)$ with  $\Sigma(i,i) = 1 ~\forall i$. It is worthwhile to note that there is an extension to Price's theorem \cite{mcmahon1964extension}, where the assumptions $\sigma^2_{ii} = 1 ~\forall i$ and $\Psi(\mathbf{x}) = \prod_i^n \Psi_i(x_i)$ are dropped; however, it only holds for the bivariate case, \ie $n = 2$, and thus is also a special case of Lemma (\ref{price_exten}).

\begin{lemma}
\label{lemma3}
For any bivariate Gaussian random variable $\mathbf{x}=[x_1,x_2]^{\top}$ $\sim$ $\mathcal{N}(\mathbf{0}_2,\Sigma_x)$, the following holds for $T(x_1,x_2) = \max(x_1,0)\max(x_2,0)$:
\begin{equation}
\begin{aligned}
&\mathbb{E}[T(x_1,x_2)] = \\
&\frac{1}{2\pi} \left(\sigma_{12} \sin^{-1}\left(\frac{\sigma_{12}}{\sigma_1 \sigma_2} \right) +  \sigma_1 \sigma_2 \sqrt{1 - \frac{\sigma_{12}^2}{\sigma_1^2 \sigma_2^2}} \right) + \frac{\sigma_{12}}{4}
\notag
\end{aligned}
\end{equation}
where $\sigma_{ij} = \Sigma_x(i,j) ~\forall i\neq j$ and $\sigma^2_i = \Sigma_x(i,i)$.
\end{lemma}

\begin{theorem}
\label{theo2}
For any function in the form of $\mathbf{g}(\mathbf{x})$ where $\mathbf{x} \sim \mathcal{N} \left(\mathbf{0}_n, \Sigma_x \right)$ and that $\mathbf{c}_1 = \mathbf{0}_p$ then: 
% \B{reviewer will ask why the mean here is zero; needs justification}:
\begin{equation}
\begin{aligned}
&\mathbb{E}[\mathbf{g}_i^2(\mathbf{x})] = \\
& 2 \sum_{v_1}^k \sum_{v_2}^{v_1-1} \mathbf{B}(i,v_1)\mathbf{B}(i,v_2) 
\left(\frac{\bar{\sigma}_{v_1 v_2}}{2\pi} \sin^{-1}\left(\frac{\bar{\sigma}_{v_1 v_2}}{\bar{\sigma}_{v_1} \bar{\sigma}_{v_2}} \right) +\right. \\
&\left.  \frac{\bar{\sigma}_{v_1} \bar{\sigma}_{v_2}}{2\pi} \sqrt{1 - \frac{\bar{\sigma}_{v_1 v_2}^2}{\bar{\sigma}_{v_1}^2 \bar{\sigma}_{v_2}^2}}  + \frac{\bar{\sigma}_{v_1 v_2}}{4} \right) + \frac{1}{2}\sum_r^k \mathbf{B}(i,r)^2 \bar{\sigma}_r^2 + \mathbf{c}_2(i)
\notag
\end{aligned} 
\end{equation}
\end{theorem}

\noindent Lastly, the variance of $\mathbf{g}(\mathbf{x})$ can be directly derived: $\text{var}(\mathbf{g}_i(\mathbf{x})) = \mathbb{E}[\mathbf{g}_i^2(\mathbf{x})] - \mathbb{E}[\mathbf{g}_i(\mathbf{x})]^2\vert_{\mu_x = \bar{\mathbf{0}}_k}$. While the previous expression assumes a zero-mean Gaussian input and bias-free first layer, \ie $\mathbf{c}_1=0$, we extend these results next to arbitrary Gaussian distributions without assumptions on $\mathbf{c_1}$. The key element here is to extend the result of Lemma \ref{lemma3}.

\begin{lemma} \label{lemma4}
For any bivariate Gaussian $\mathbf{x} \sim \mathcal{N}\left(\mu_x,\Sigma_x \right)$, where $\mu_x =[\mu_1, \mu_2]^\top$ and $\Sigma = \begin{bmatrix}
    \sigma_1^2  & \rho\sigma_1\sigma_2 \\
    \rho\sigma_1\sigma_2  & \sigma_2^2
\end{bmatrix}$, then we have that

\begin{equation}
\label{eq:big_result}
\begin{aligned}
&\mathbb{E}[\max({x_1},{0}) \max({x_2},{0})] =  \Omega(\mu_1,\mu_2,\sigma_1,\sigma_2,\rho)  \\
+
&\begin{cases}
\frac{\mu_1\mu_2+\rho\sigma_1\sigma_2}{\pi} \left(I_{a_1,b_1}\left(\infty\right) - I_{a_1,b_1}\left(\frac{-\mu_2}{\sqrt{2}\sigma_2}\right)\right),
\\
~~~~~~~~~~~~~~~~~~~~~~~~
\text{for \,} |\rho|<\frac{1}{\sqrt{2}},\\

\frac{\mu_1\mu_2+\rho\sigma_1\sigma_2}{\pi} \Biggl[ \frac{\pi}{4} \text{sign}(\rho) + \frac{\pi}{4}\text{erf}(\frac{\mathbf{e}_1^\top \tilde{\Sigma}\mu_x}{\sqrt{2}}) \text{erf}(\frac{\mu_2}{\sqrt{2}\sigma_2}) \\
- \text{sign}(\rho) \left(
I_{a_2,b_2}\left(\infty\right)-I_{a_2,b_2}\left(\text{sign}(\rho)\frac{\mathbf{e}_1^\top\tilde{\Sigma}\mu_x}{\sqrt{2}}\right)\right) \Biggr], \\
~~~~~~~~~~~~~~~~~~~~~~~~ \text{for \,} |\rho|>\frac{1}{\sqrt{2}}.
\end{cases}
\end{aligned}
\end{equation}

where 
\begin{equation}
\begin{aligned}
\label{omega_expression}
    \Omega &=  \frac{\sqrt{|\Sigma|}}{2\pi}  \exp \left(-\frac{1}{2}
\mu_x^\top \Sigma^{-1} \mu_x
\right) + \frac{\mu_1\sigma_2}{2\sqrt{2\pi}} 
\exp\left(\frac{-\mu_2^2}{2\sigma_2^2}\right) \\
&\left( 1 + \text{erf}\biggl(\frac{\mathbf{e}_1^\top \tilde{\Sigma}\mu_x}{\sqrt{2}}\biggr) \right) 
+ \frac{\mu_2\sigma_1}{2\sqrt{2\pi}} \exp\left(\frac{-\mu_1^2}{2\sigma_1^2}\right)\\
&\Biggl( 1 +  \text{erf}\biggl(\frac{\mathbf{e}_2^\top \tilde{\Sigma}\mu_x}{\sqrt{2}}\biggr) \Biggr) + \frac{\mu_1\mu_2+\rho\sigma_1\sigma_2}{4} \left( 1 + \text{erf}\left(\frac{\mu_2}{\sqrt{2}\sigma_2}\right) \right),
\end{aligned}
\end{equation}

and where
\begin{equation}
\begin{aligned}
\label{integration_identity}
&I_{a,b}(x) = \frac{\pi}{4}\text{erf}(x) \text{erf}(b) + \frac{\sqrt{\pi}}{2} \exp(-b^2)\\
& \sum_{u=0}^{\infty} \Bigg\{ \frac{(a/2)^{2u+1}}{\Gamma(u+\nicefrac{3}{2})} P(u+1,x^2)H_{2u}(b)\\
& - \frac{\text{sign}(x)(a/2)^{2u+2}}{\Gamma(u+2)}P(u+\nicefrac{3}{2},x^2)H_{2u+1}(b) \Bigg\}.
\end{aligned}
\end{equation}

Note that $\mathbf{e}_1$ and $\mathbf{e}_2$ are the two dimensional canonical vectors. Moreover, note that $\tilde{\Sigma} = \text{Diag}([\nicefrac{\sqrt{|\Sigma|}}{\sigma_2},\nicefrac{\sqrt{|\Sigma|}}{\sigma_1}])\Sigma^{-1}$ where $\text{Diag}(\mathbf{v})$ rearranges the elements of the vector $\mathbf{v}$ into a diagonal matrix and $|\Sigma|$ denotes the matrix determinant. The constants $a_1,b_1,a_2,$ and $b_2$ are $\nicefrac{\rho}{\sqrt{1-\rho^2}}$, $\nicefrac{\mu_1\sigma_2}{\sqrt{2|\Sigma|}}$, $\nicefrac{\sqrt{1-\rho^2}}{|\rho|}$ and $\nicefrac{-\mu_1}{\sqrt{2}\rho\sigma_1}$, respectively. Lastly, $H(.)$ is the Hermite polynomial, $P(.,.)$ is the normalized incomplete Gamma function and $\Gamma(.)$ is the standard Gamma function.
\end{lemma}
\begin{proof}
This is a sketch of the proof.

\begin{equation}
\begin{aligned}
\label{first_integration}
I_0 &= \mathbb{E}[\max(x_1,0)\max(x_2,0)] \\
& = \int_{0}^{\infty} \int_{0}^{\infty} x_1 x_2 f_{X_1,X_2}(x_1,x_2) dx_1dx_2 \\
& = \int_{0}^{\infty} x_2 f_{X_2}(x_2) \int_{0}^{\infty} x_1 f_{X_1|X_2}(x_1|x_2) dx_1dx_2 \\
& = \frac{1}{\sigma_2} \int_{0}^{\infty} x_2 f_{X_2}(x_2) \Bigg [\frac{\sqrt{|\Sigma|}}{2 \pi} \exp\left(\frac{r^2(x_2)}{2 |\Sigma|}\right) \\
& + r(x_2) \Phi\left(\frac{r(x_2)}{\sqrt{|\Sigma|}}\right)  \Bigg] dx_2,
\end{aligned}
\end{equation}
where $r(x_2) = \mu_1\sigma_2 +\rho \sigma_1(x_2-\mu_2)$. The functions $f_{X_1,X_2}$, $f_{X_1|X_2}$ and $f_{X_2}$ are the joint bivariate, conditional and marginal Gaussian distributions. By integration by parts, Leibniz's rule, some identities and substitutions, Equation \eqref{first_integration} reduces to:

\begin{equation}
\begin{aligned}
\label{eq:trouble_sum_integral}
I_0 &= \Omega(\mu_1,\mu_2,\sigma_1,\sigma_2,\rho) \\
& + \frac{\mu_1\mu_2+\rho\sigma_1\sigma_2}{2\sqrt{\pi}} \underbrace{\int_{\frac{-\mu_2}{\sqrt{2}\sigma_2}}^{\infty} \exp(-z^2) \text{erf}\left(\frac{\rho z + \nicefrac{\mu_1}{\sqrt{2}\sigma_1}
}{\sqrt{(1-\rho^2})}\right)}_{\int_{\kappa}^\infty\phi(z)dz} dz,
\end{aligned}
\end{equation}

\noindent where $\Omega(\mu_1,\mu_2,\sigma_1,\sigma_2,\rho)$ is given by Equation \eqref{omega_expression}. As for the remaining integral, we exploit identities (2.1) and (2.2) in \cite{fayed2014evaluation}, which states that $I_{a,b}(x) = \frac{\sqrt{\pi}}{2} \int_{0}^{x} \exp(-t^2)\text{erf}(at+b)dt$ has a closed form solution given in Equation \eqref{integration_identity}. Thus, one can represent the integral in Equation (\ref{eq:trouble_sum_integral}) as $\frac{\sqrt{\pi}}{2}\int_{\kappa}^{\infty} \phi(z) dz = I_{a,b}(\infty) - I_{a,b}(\kappa)$  where $a = \rho / \sqrt{1-\rho^2}$, $b =  \mu_1 / (\sqrt{2}\sigma_1 \sqrt{1-\rho^2})$ and $\kappa = -\mu_2 / \sqrt{2}\sigma_2$. Now note that the infinite series corresponding to $I_{a,b}(x)$ and $I_{a,b}(\infty)$ in Equation \eqref{integration_identity} converges when $|a| < 1$ or equivalently $|\rho| < 1 / \sqrt{2}$ which proves the first case in Equation \eqref{eq:big_result}. As for the case $|\rho| > 1 / \sqrt{2}$, by integrating the integral in Equation \eqref{eq:trouble_sum_integral} by parts, we have

\begin{equation}
\begin{aligned}
\label{eq:other_integration_by_parts}
&\frac{\sqrt{\pi}}{2} \int_{\kappa}^{\infty} \exp(-t^2)\text{erf}(at+b)dt = - \frac{\pi}{4}\text{erf}\left(a\kappa+b\right)\text{erf}(\kappa) \\
&+ \text{sign}(a) \left(\frac{\pi}{4}  -   I_{\nicefrac{1}{|a|}, \nicefrac{-b}{a}}(\infty) + I_{\nicefrac{1}{|a|}, \nicefrac{-b}{a}}\left(\left(a \kappa + b\right)\text{sign}(a)\right)\right).
\end{aligned}
\end{equation}

\noindent Note that the series from the identity replacing $I_{\nicefrac{1}{|a|},\nicefrac{-b}{a}}$ converges when $|a|>1$ or equivalently $|\rho| > 1/\sqrt{2}$. Thus, substituting this result back in Equation \eqref{eq:trouble_sum_integral} derives the second case of Equation \eqref{eq:big_result} and completing the proof.
\end{proof}

\noindent Following Theorem \ref{theo2}, a closed form expression for $\mathbb{E}[\mathbf{g}_i^2(\mathbf{x})]$ under generic Gaussian distributions can be derived by substituting the result from Lemma \ref{lemma4} (in lieu of Lemma \ref{lemma3}) in the proof of Theorem \ref{theo2} deriving an expression for the variance of $\mathbf{g}_i(\mathbf{x})$. Moreover, we show in the \textbf{Appendix} that Equation \eqref{eq:big_result} recovers Lemma \ref{lemma3} for when $\mu_1 = \mu_2 = 0$.

\subsection{Extension to Deeper PL-DNNs} \label{sec:extension_2_deeper_networks}
\noindent To extend the previous results to deeper DNNs that are not in the form (Affine, ReLU, Affine), we first denote the larger DNN as $\mathbf{R}(\mathbf{x}) :\mathbb{R}^n \rightarrow \mathbb{R}^d$ (\eg a mapping of the input to the logits of $d$ classes). By choosing the $l^{\text{th}}$ ReLU layer, any $\mathbf{R}(.)$ can be decomposed into: $\mathbf{R}(\mathbf{x}) = \mathbf{R}_{l+1} \left(\text{ReLU}_{l} \left(\mathbf{R}_{l-1}\left( \mathbf{x}\right) \right) \right)$. In this paper, we employ a simple two-stage linearization based on Taylor series approximation to cast  $\mathbf{R}(.)$  into the form (Affine, ReLU, Affine). For example, we can linearize it around points $\mathbf{x}$ and $\mathbf{y} = \text{ReLU}_{l}\left(\mathbf{R}_{l-1}(\mathbf{x}) \right)$, such that $\mathbf{R}_{l-1}(\mathbf{x}) \approx \mathbf{A}\mathbf{x} + \mathbf{c}_1$ and $\mathbf{R}_{l+1}(\mathbf{y}) \approx \mathbf{B}\mathbf{y} + \mathbf{c}_2$. The resulting function after linearization is $\mathbf{R}(\mathbf{x}) \approx \mathbf{B} \text{ReLU}_{l}\left(\mathbf{A} \mathbf{x} + \mathbf{c}_1\right) + \mathbf{c}_2$. Figure \ref{pull_fig} shows this two-stage linearization. Details in regards to the selection of the layer of linearization $l$ and the points of linearization are discussed thoroughly next.

\section{Experiments}\label{exps}
\noindent In this section, we discuss a variety of experiments to provide the following insights. \textbf{(i)} Although the derived output variance of the  Affine-ReLU-Affine network based on Equation \eqref{eq:big_result} is impractical, the infinite sum can be accurately approximated with as few as 20 terms leading to an efficient computation. \textbf{(ii)} We conduct several controlled experiments to investigate the choice of the linearization layer $l$, at which two-stage linearization is performed. We also validate the tightness of both the first and second moment expressions for deeper networks under different linearization points, as well as, showing that the new derived variance based on Lemma \ref{lemma4} is much tighter than the one based on Lemma \ref{lemma3} for general input Gaussian distributions. \textbf{(iii)} Lastly, extensive experiments on MNIST and Emotion datasets validate that our derived expressions can be used to construct targeted and non-targeted adversarial Gaussian attacks. In particular, and following the recent successes of sparse pixel attacks \cite{modas019sparsefool}, we demonstrate that our expressions can indeed be utilized to design sparse and smooth Gaussian perturbations leading to perceptually feasible input attacks.

\subsection{On the Efficacy of Approximating Equation \eqref{eq:big_result}}
Computing the variance of the Affine-ReLU-Affine network, \ie $\mathbf{g}(\mathbf{x})$, under general Gaussian input $\mathbf{x}$, as per Equation \eqref{eq:big_result} in Lemma \ref{lemma4}, requires the evaluation of Equation \eqref{integration_identity}, which is impractical as it involves an infinite series. We show here that the series can be sufficiently well approximated with as few as 20 terms. To demonstrate this along with the sensitivity of Equation \eqref{eq:big_result} to $\mu_1$, $\mu_2$, $\sigma_1$, $\sigma_2$ and $\rho$, we report the \textit{maximum absolute error} between the Monte Carlo estimates of $\mathbb{E}[\max(x_1,0)\max(x_2,0)]$ and truncated versions of the sum in Equation \eqref{eq:big_result} with $1$, $5$, $10$, $20$, and $50$ terms over a grid of all combinations of the five arguments. In particular, $\mu_1$ and $\mu_2$ are sampled uniformly from the grid $[-2,2]$, $\sigma_1$ and $\sigma_2$ are on the uniform grid $[0.2,2]$, and lastly $\rho$ is sampled uniformly from the grid $[-0.7,0.7]$, where all parameters are sampled with $0.2$ spacing. In addition, we also include $\rho = 0$ and $\rho = 0.999$. Figure \ref{fig:eff_of_apprx_1} reports the \textit{maximum absolute error} of all possible combinations of the aforementioned parameters in log-scale with an increasing number of terms of Equation \eqref{integration_identity}. We observe from Figure \ref{fig:eff_of_apprx_1} that, with as few as 20 terms, the \textit{maximum absolute error} between the Monte Carlo estimates and the truncated version of Equation \eqref{eq:big_result} is $10^{-2.5} \approx 0.003$. This occurs regardless of the choice of $\mu_1$, $\mu_2$, $\sigma_1$ and $\sigma_2$ and particularly when $\rho$ is close to $\pm \nicefrac{1}{\sqrt{2}} \approx \pm0.7$, which is the disjunction in Equation \eqref{eq:big_result}. Recall that the disjunction occurs at these values of $\rho$, since the infinite series diverges in such cases. On the other hand, the \textit{maximum absolute error} decreases rapidly so long as $\rho$ is away from $\pm \nicefrac{1}{\sqrt{2}}$. Now that Equation \eqref{eq:big_result} can reliably and efficiently be approximated with a small number of terms, deeming it efficient, the closed form expression of Equation \eqref{eq:big_result} can be used to compute the output variance of $\mathbf{g}(\mathbf{x})$ for various applications. Throughout all remaining experiments, we will use only 5 terms, since the absolute error is of order $\approx 10^{-4}$ for all choices of $\mu_1,\mu_2,\sigma_1,\sigma_2,\rho$ except for the improbable two singularities $\rho \in \{\pm\nicefrac{1}{\sqrt{2}}\}$.

\begin{figure}[t]
    \centering
    \includegraphics[width=0.45\textwidth]{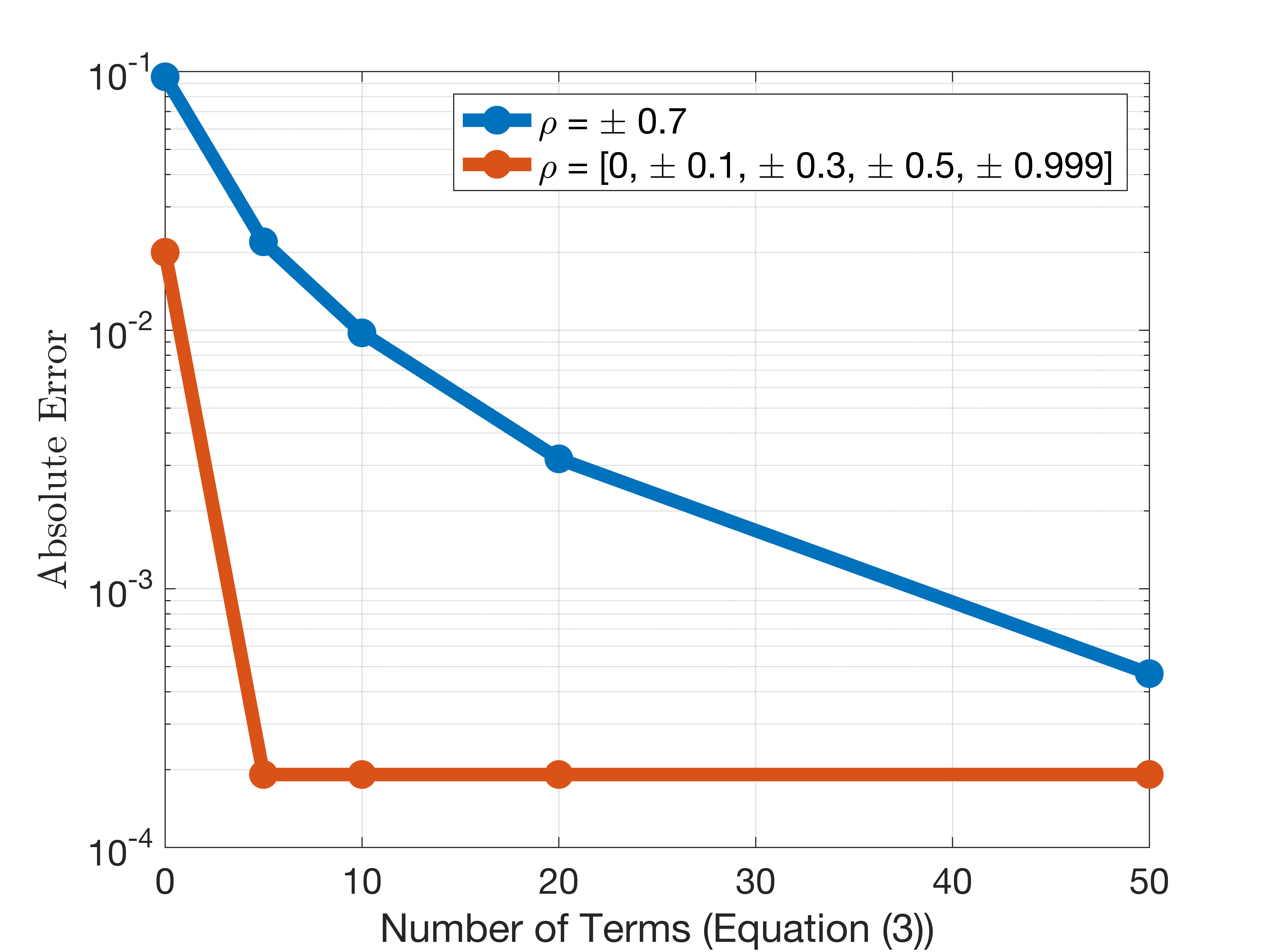}
    \caption{\textbf{Approximating Equation \eqref{eq:big_result}.} The \textit{maximum absolute error} between the Monte Carlo estimates and the truncated Equation \eqref{eq:big_result} for different parameterizations of the bivariate Gaussian decreases rapidly as the number of terms in the truncated sum increases. The two plots show that the error decreases very quickly and regardless of all the other parameters, when $\rho$ is not close to the disjunction, \ie $\rho = \pm \nicefrac{1}{\sqrt{2}}$.}
    \label{fig:eff_of_apprx_1}
\end{figure}

\subsection{Tightness of Network Moments} \label{sec:tightness_NM}
\vspace{4pt}\noindent \textbf{Choice of the Two-Stage Linearization Layer $l$}. The derived expressions for the first and second moments are for a small network in the form Affine-ReLU-Affine. As detailed in Subsection \ref{sec:extension_2_deeper_networks}, such results can be extended and applied to deeper networks through the proposed two-stage linearization. However, it is not clear how to choose the layer of linearization $l$. This subsection addresses this design choice by conducting an ablation to study the impact of varying $l$. In particular, we show that there is an intrinsic trade off between memory efficiency and linearization error for the choice of the layer $l$, around which two-stage linearization is performed. To illustrate this, consider the following network $\mathbf{R}(\mathbf{x}) = \mathbf{R}_{l+1}\left(\text{ReLU}(\mathbf{R}_{l-1}(\mathbf{x}))\right)$ where $\mathbf{R}(.): \mathbb{R}^n \rightarrow \mathbb{R}^d$, $\mathbf{R}_{l-1}(.) : \mathbb{R}^n \rightarrow \mathbb{R}^{q}$, and $\mathbb{R}_{l+1}: \mathbb{R}^{q} \rightarrow \mathbb{R}^{d}$. Performing two-stage linearization requires the memory of storing the Jacobians of the two-stage linearization $\nabla \mathbf{R}_{l-1}(.)$ and $\nabla \mathbf{R}_{l+1}(.)$, which is a total of $2q+n+d$ elements. When  $l$ is chosen to be small (early convolutional layers), the value $q$ is usually very large, as it is the total number of pixels across all feature maps. Meanwhile, when $l$ is large, $q$ is usually only the number of nodes in a fully connected layer. However, the choice of large $l$ in general leads to larger linearization error. To demonstrate this, we conduct experiments on the LeNet architecture \cite{lecun1999object} pretrained on the MNIST digit dataset \cite{lecun1998mnist}. Note that LeNet has a total of four layers, two of which are convolutional with max pooling and the other two are fully connected. We perform two-stage linearization on LeNet with a varying choice of $l$, where we compare the $\ell_2$ difference between the prediction scores of LeNet and the two-stage linearized version with the point of linearization taken to be a noisy version of a random image from the MNIST validation set. Table \ref{linearizations} demonstrates that the choice of smaller $l$ is best, in $\ell_2$ sense, for the two-stage linearization across all the various levels of noisy versions of the input. This implies a trade off between memory efficiency (better memory complexity with larger $l$) and accuracy  (better $\ell_2$ linearization error for smaller $l$). Therefore and due to memory constraints, $l$ is chosen to be the fully-connected layer just before the last ReLU activation in all experiments, unless stated otherwise. 

\begin{table}[t]
\footnotesize
\centering
\caption{\textbf{Varying the layer of linearization $l$.} The average approximation error, on a randomly sampled MNIST image corrupted with additive Gaussian noise, between LeNet and the two stage-linearized version increases as the layer of linearization $l$ increases.}
\begin{tabular}{c|c|c|c|c|c}
\small
Noise & $\pm$ 0.5 & $\pm$ 0.75 & $\pm$ 1 & $\pm$ 1.5 & $\pm$ 2 \\
\hline
$l=1$ & \textbf{0.0241} & \textbf{0.0362} & \textbf{0.0485} & \textbf{0.0730} & \textbf{0.0977} \\
$l=2$ & 0.0330 & 0.0497 & 0.0663 & 0.0996 & 0.1330 \\
$l=3$ & 0.0329 & 0.0495 & 0.0661 & 0.0993 & 0.1327 \\
\hline
\end{tabular}
\label{linearizations}
\end{table}

\vspace{4pt}\noindent \textbf{Tightness of Moment Expressions on LeNet.} It is conceivable that the two-stage linearization might impact the tightness of the derived moment expressions when applied to deeper real PL-DNNs. Here, we empirically study their tightness by comparing them against Monte Carlo estimates over $10^4$ samples on LeNet. Using the MNIST dataset, the input to the network is $\mathbb{R}^{28 \times 28}$ with 10 output classes (\ie $d=10$). In this case, following Section \ref{sec:extension_2_deeper_networks}, the two-stage linearization is performed such that $\mu_x = \mathbf{M}$, $\mathbf{y} = \mathbf{R}_{l-1}(\mathbf{x})$ and $l = 3$ for memory efficiency, where $\mathbf{M}$ is an image selected from the MNIST testing set. Thus, the input is $\mathbf{x} \sim \mathcal{N}\left(\mathbf{M},\Sigma \right)$ where we randomly generate a covariance matrix such that $\text{trace}(\Sigma) = \sigma^2(28 \times 28)$ with reasonable noise levels $\sigma \in \{0.001, 0.01, 0.05, 0.1\}$ when $\mathbf{M} \in [0,1]^{28 \times 28}$. Since the LeNet architecture has $d = 10$, we report the tightness of the analytic mean from Theorem \ref{theo1}, variance from Theorem \ref{theo2}, and the new general variance expression based on Lemma \ref{lemma4} for $\mathbf{g}_i(\mathbf{x})~\forall i$. As for the metric, we report the average \textit{absolute relative difference} $E_r(x, y) = \frac{2|x - y|}{|x| + |y|}$ of the analytic mean and variance expressions (Theorems \ref{theo1} and \ref{theo2}) to their Monte Carlo counterparts. We refer to each as $\mathbb{E}[\mathbb{E}_{error}] = \mathbb{E}\left[E_r(\mathbb{E}[\mathbf{g}_i(\mathbf{x})], \mathbb{E}_{MC})\right]$ and $\mathbb{E}[\text{var}_{error}] = \mathbb{E}\left[E_r(\text{var}({g}_i(\mathbf{x})), \text{var}_{MC})\right]$, respectively. Similarly, we refer to the error of the Monte Carlo estimates to the new variance expression based on Lemma \ref{lemma4} as $\mathbb{E}[\text{var}^{\text{new}}_{error}]$, where we find that the summation in Equation \eqref{integration_identity} can be truncated to only $5$ terms without scarifying much accuracy. We average the results over the complete MNIST test set. We report the tightness results across all classes in Table \ref{lenet_exactnesss}, where the closer the errors are to $0$ the better. For instance, at $\sigma = 0.05$, the \textit{absolute relative difference} for the mean expression of Theorem \ref{theo1} are close to $0.1$, \ie $\mathbb{E}[\mathbb{E}_{error}] \approx 0.11$. That is to say, the mean expression is tight even though two-stage linearization is being performed on a real network. Whereas, the variance expression of Theorem \ref{theo2} is less accurate, $\mathbb{E}[\text{var}_{error}] \approx 0.19$, and this can be attributed to the assumptions that do not hold (zero-mean input Gaussian and $\mathbf{c}_1=0$). On the other hand, the new general expression for the output variance based on Lemma \ref{lemma4} is significantly much tighter than the one from Theorem \ref{theo2}, as the errors compared to the Monte Carlo estimates are closer to $0$, \ie $\mathbb{E}[\text{var}_{error}^{\text{new}}] \approx 0.07$. This shows that our new variance expression is far tighter and less sensitive to two-stage linearization despite the truncation of the infinite series to as few as $5$ terms. Furthermore, complementing the results in Table \ref{lenet_exactnesss} and instead of reporting the \textit{absolute relative difference} alone, we visualize the histogram of LeNet output variances for all testing MNIST images under varying noise levels in Table \ref{histogram_plots_variance} for better interpretability of the results.
% Finally, we show in Table \ref{histogram_plots_variance} a histogram of the Monte Carlo estimates of the variance compared against both the variance expression based on \cite{bibi2018analytic} and our new variance expression.
% Please, refer to the \textbf{Appendix} for distribution visualizations of both variance expressions compared with the Monte Carlo estimates.

\begin{table}
\centering
\caption{\textbf{Tightness results across all classes on MNIST}. Using different values of input noise $\sigma$, the table shows that the mean expression of Theorem \ref{theo1} is tight and insensitive under two-stage linearization. Moreover, the table demonstrates that the new general variance expression based on Lemma \ref{lemma4} is far tighter, despite the truncation of the infinite series, compared the previous results from Theorem \ref{theo2}.}
\scalebox{0.95}{
\begin{tabular}{c|c|c|cc}
\toprule
$\sigma$ & $\mathbf{g}_i(\mathbf{x})$ & $\mathbb{E}[\mathbb{E}_{error}]$ & $\mathbb{E}[\text{var}_{error}]$ & $\mathbb{E}[\text{var}_{error}^{\text{new}}]$ \\
\midrule
0.001 & 0 &  $0.002 \pm 0.029$ &  $0.197 \pm 0.126$ &  $\mathbf{0.102 \pm 0.070}$ \\
      & 1 &  $0.001 \pm 0.011$ &  $0.195 \pm 0.148$ &  $\mathbf{0.065 \pm 0.051}$ \\
      & 2 &  $0.002 \pm 0.024$ &  $0.178 \pm 0.127$ &  $\mathbf{0.063 \pm 0.051}$ \\
      & 3 &  $0.004 \pm 0.038$ &  $0.255 \pm 0.190$ &  $\mathbf{0.146 \pm 0.118}$ \\
      & 4 &  $0.002 \pm 0.019$ &  $0.192 \pm 0.143$ &  $\mathbf{0.131 \pm 0.100}$ \\
      & 5 &  $0.002 \pm 0.022$ &  $0.157 \pm 0.115$ &  $\mathbf{0.075 \pm 0.067}$ \\
      & 6 &  $0.002 \pm 0.024$ &  $0.248 \pm 0.161$ &  $\mathbf{0.137 \pm 0.088}$ \\
      & 7 &  $0.002 \pm 0.034$ &  $0.225 \pm 0.157$ &  $\mathbf{0.120 \pm 0.101}$ \\
      & 8 &  $0.005 \pm 0.036$ &  $0.173 \pm 0.122$ &  $\mathbf{0.103 \pm 0.082}$ \\
      & 9 &  $0.003 \pm 0.030$ &  $0.182 \pm 0.132$ &  $\mathbf{0.057 \pm 0.057}$ \\
\cmidrule{2-5}& Avg &  $0.003 \pm 0.027$ &  $0.200 \pm 0.142$ &  $\mathbf{0.100 \pm 0.078}$ \\
\midrule
0.010 & 0 &  $0.021 \pm 0.098$ &  $0.202 \pm 0.127$ &  $\mathbf{0.026 \pm 0.021}$ \\
      & 1 &  $0.014 \pm 0.067$ &  $0.197 \pm 0.150$ &  $\mathbf{0.022 \pm 0.017}$ \\
      & 2 &  $0.023 \pm 0.104$ &  $0.176 \pm 0.126$ &  $\mathbf{0.029 \pm 0.022}$ \\
      & 3 &  $0.034 \pm 0.138$ &  $0.252 \pm 0.190$ &  $\mathbf{0.024 \pm 0.019}$ \\
      & 4 &  $0.022 \pm 0.111$ &  $0.190 \pm 0.139$ &  $\mathbf{0.024 \pm 0.020}$ \\
      & 5 &  $0.019 \pm 0.099$ &  $0.163 \pm 0.117$ &  $\mathbf{0.026 \pm 0.020}$ \\
      & 6 &  $0.017 \pm 0.074$ &  $0.255 \pm 0.163$ &  $\mathbf{0.024 \pm 0.019}$ \\
      & 7 &  $0.012 \pm 0.075$ &  $0.227 \pm 0.155$ &  $\mathbf{0.024 \pm 0.020}$ \\
      & 8 &  $0.041 \pm 0.149$ &  $0.178 \pm 0.125$ &  $\mathbf{0.027 \pm 0.022}$ \\
      & 9 &  $0.025 \pm 0.116$ &  $0.186 \pm 0.134$ &  $\mathbf{0.024 \pm 0.019}$ \\
\cmidrule{2-5}& Avg &  $0.023 \pm 0.103$ &  $0.203 \pm 0.143$ &  $\mathbf{0.025 \pm 0.020}$ \\
\midrule
0.050 & 0 &  $0.101 \pm 0.212$ &  $0.177 \pm 0.118$ &  $\mathbf{0.075 \pm 0.051}$ \\
      & 1 &  $0.083 \pm 0.166$ &  $0.189 \pm 0.145$ &  $\mathbf{0.059 \pm 0.042}$ \\
      & 2 &  $0.125 \pm 0.259$ &  $0.183 \pm 0.131$ &  $\mathbf{0.071 \pm 0.049}$ \\
      & 3 &  $0.146 \pm 0.320$ &  $0.248 \pm 0.179$ &  $\mathbf{0.068 \pm 0.049}$ \\
      & 4 &  $0.111 \pm 0.244$ &  $0.186 \pm 0.139$ &  $\mathbf{0.059 \pm 0.045}$ \\
      & 5 &  $0.099 \pm 0.241$ &  $0.159 \pm 0.116$ &  $\mathbf{0.065 \pm 0.047}$ \\
      & 6 &  $0.084 \pm 0.159$ &  $0.209 \pm 0.148$ &  $\mathbf{0.079 \pm 0.051}$ \\
      & 7 &  $0.061 \pm 0.163$ &  $0.225 \pm 0.167$ &  $\mathbf{0.076 \pm 0.055}$ \\
      & 8 &  $0.157 \pm 0.331$ &  $0.161 \pm 0.116$ &  $\mathbf{0.073 \pm 0.054}$ \\
      & 9 &  $0.119 \pm 0.260$ &  $0.179 \pm 0.131$ &  $\mathbf{0.075 \pm 0.049}$ \\
\cmidrule{2-5}& Avg &  $0.109 \pm 0.235$ &  $0.192 \pm 0.139$ &  $\mathbf{0.070 \pm 0.049}$ \\
\midrule
0.100 & 0 &  $0.177 \pm 0.257$ &  $\mathbf{0.159 \pm 0.121}$ &  $0.165 \pm 0.085$ \\
      & 1 &  $0.176 \pm 0.223$ &  $0.194 \pm 0.150$ &  $\mathbf{0.138 \pm 0.073}$ \\
      & 2 &  $0.265 \pm 0.367$ &  $0.209 \pm 0.151$ &  $\mathbf{0.139 \pm 0.079}$ \\
      & 3 &  $0.243 \pm 0.410$ &  $0.263 \pm 0.178$ &  $\mathbf{0.149 \pm 0.083}$ \\
      & 4 &  $0.232 \pm 0.344$ &  $0.200 \pm 0.155$ &  $\mathbf{0.126 \pm 0.078}$ \\
      & 5 &  $0.189 \pm 0.321$ &  $0.165 \pm 0.125$ &  $\mathbf{0.131 \pm 0.077}$ \\
      & 6 &  $0.158 \pm 0.211$ &  $\mathbf{0.161 \pm 0.123}$ &  $0.178 \pm 0.085$ \\
      & 7 &  $0.130 \pm 0.225$ &  $0.251 \pm 0.197$ &  $\mathbf{0.172 \pm 0.091}$ \\
      & 8 &  $0.251 \pm 0.415$ &  $\mathbf{0.161 \pm 0.120}$ &  $0.163 \pm 0.089$ \\
      & 9 &  $0.231 \pm 0.355$ &  $0.195 \pm 0.143$ &  $\mathbf{0.167 \pm 0.083}$ \\
\cmidrule{2-5}& Avg &  $0.205 \pm 0.313$ &  $0.196 \pm 0.146$ &  $\mathbf{0.153 \pm 0.082}$ \\
\bottomrule
\end{tabular}
}
\label{lenet_exactnesss}
\end{table}

\begin{table*}[!t]
\centering
\caption{\textbf{LeNet variance histograms on MNIST}. The table shows the histogram of LeNet output variances for all testing MNIST images under varying noise levels. We compare the estimation of output variance through Monte-Carlo sampling of $10^4$ instances against the variance expressions in Theorem \ref{theo2} (Old) and Lemma \ref{lemma4} (New). We also report in the legend the averaged \emph{absolute relative difference} over the images.} \label{fig:variance_histograms_low}
\scalebox{0.9}{
\begin{tabular}{c}
    \includegraphics[width=\linewidth]{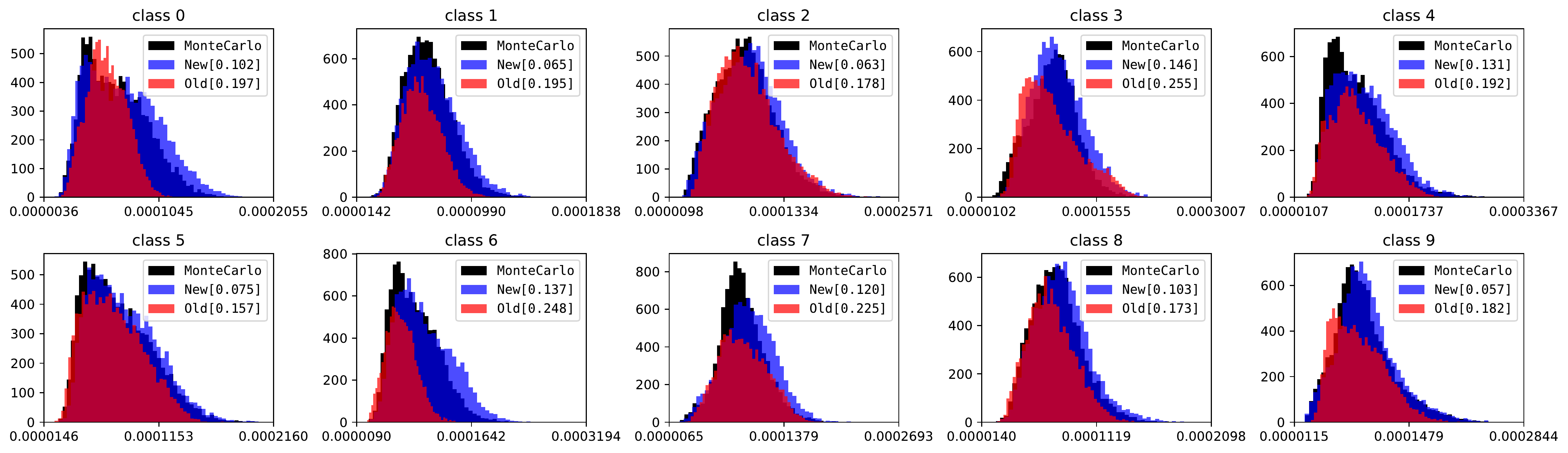}\\
    Variance histograms with input noise $\sigma = 0.001$\\\rule{0pt}{4ex}\\
    \includegraphics[width=\linewidth]{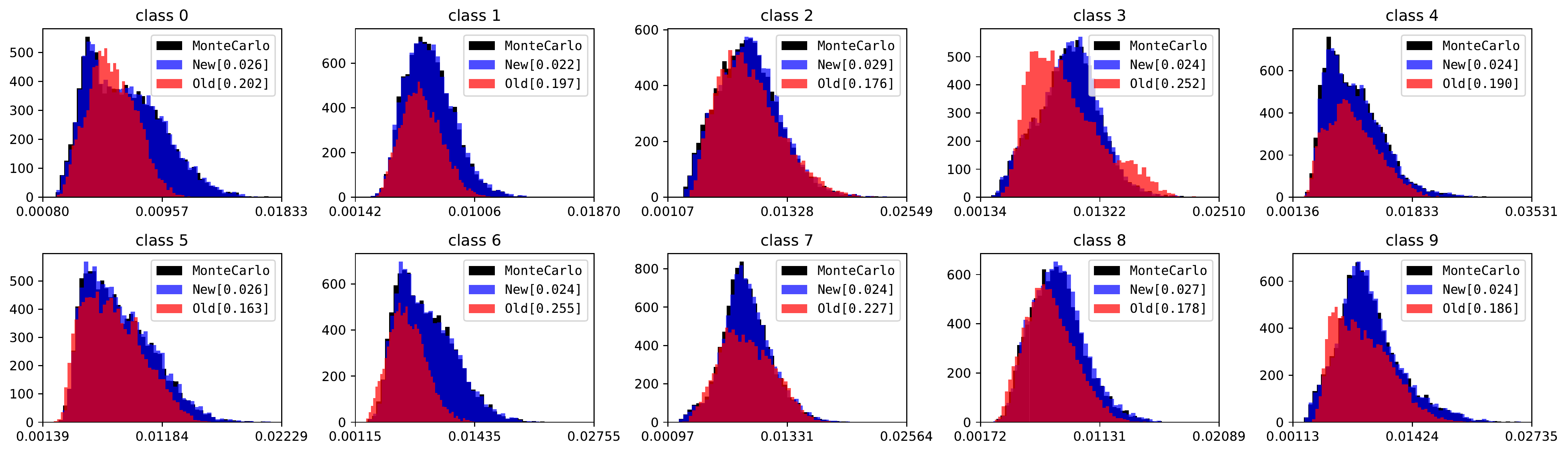}\\
    Variance histograms with input noise $\sigma = 0.01$\\\rule{0pt}{4ex}\\
    \includegraphics[width=\linewidth]{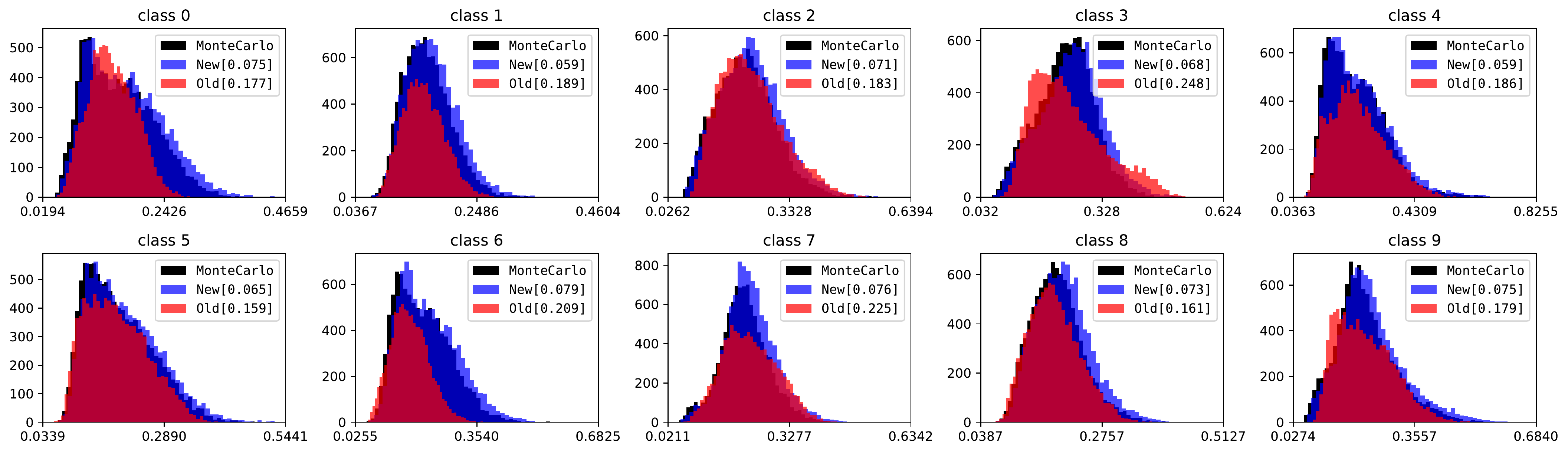}\\
    Variance histograms with input noise $\sigma = 0.05$\\\rule{0pt}{4ex}\\
    \includegraphics[width=\linewidth]{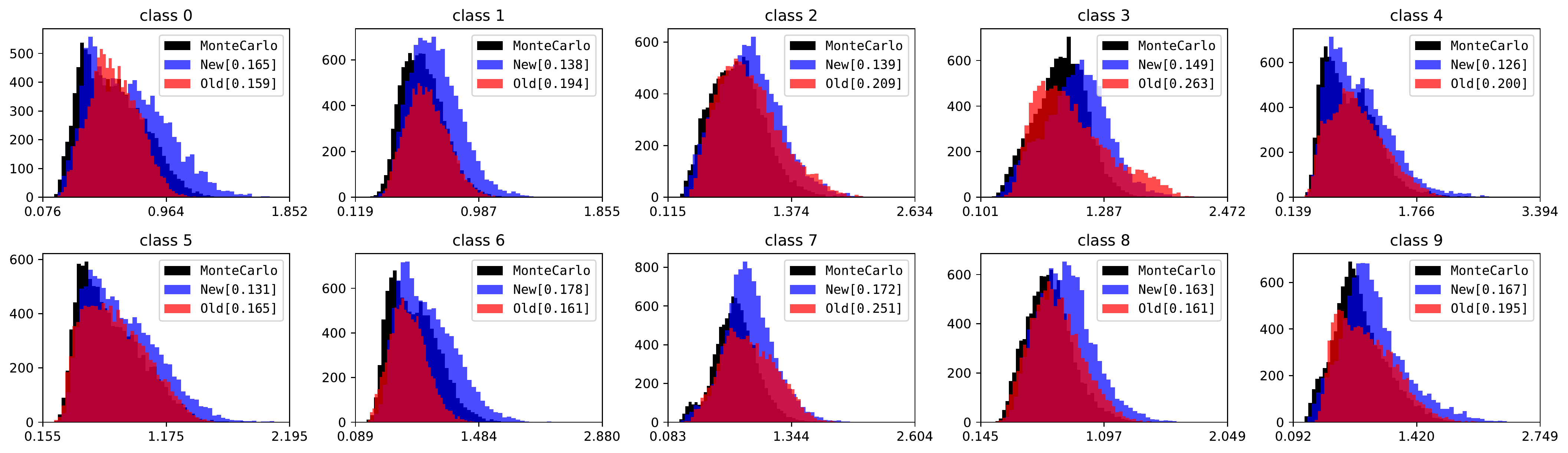}\\
    Variance histograms with input noise $\sigma = 0.1$\\
\end{tabular}
}
\label{histogram_plots_variance}
\end{table*}

\vspace{4pt}\noindent \textbf{Sensitivity to the Point of Linearization.} In all previous tightness validation experiments of the moment expressions, the point at which two-stage linearization is performed was restricted to be the input image, \ie $\mathbf{M}$. Clearly, this strategy suffers from limited scalability, since analyzing the output moment expressions of deep networks over a large dataset requires performing the expensive two-stage linearization for every image in the dataset. To circumvent this difficulty, we study the sensitivity of the tightness of the expressions under two-stage linearization around only a small set of input images from the dataset. That is to say, we choose a set of representative input images, at which the two-stage linearization parameters $\mathbf{A}, \mathbf{B}, \mathbf{c}_1$ and $\mathbf{c}_2$ are computed only once and offline for each input image. Now, to evaluate the network moments for an unseen input, we simply use the two-stage linearization parameters of the \emph{closest} linearization point to this input.

In this experiment, we study the tightness of our expressions under this more relaxed linearization strategy using LeNet on the MNIST testing set. We cluster the images in the testing dataset using $k$-means on the image intensity space with different values of $k$. We use the cluster centers as the linearization points. Table \ref{linearization_sensitivity} summarizes the tightness of the expressions for $k \in \{250, 500, 1000, 2500, 5000, 10000\}$ and compares them against a weak baseline, where the linearization point is set to be the \emph{farthest} image in each cluster from the cluster center with $k \in \{250, 500\}$. It is clear that the new variance expression based on Lemma \ref{lemma4} remains very close to the Monte Carlo estimate across different number of linearization points $k$, even when $k$ is as low as $250$, \ie only $2.5\%$ of the testing set. On the other hand, the analytic variance derived from Theorem \ref{theo2} is less accurate but stays within an acceptable range with $\mathbb{E}[\text{var}_{error}] \approx 0.25$. This indeed reaffirms that even upon truncating the infinite series in Equation \eqref{integration_identity} to only $5$ terms, the new variance expression is much tighter and more accurate under network linearization than the preliminary result of Theorem \ref{theo2} in \cite{bibi2018analytic}. As for the analytic mean, however, it is more sensitive to the point of linearization but even in the worst case, \ie $k = 250$ and $\sigma = 0.1$ for example, the average error $\mathbb{E}[\mathbb{E}_{error}]$ doesn't exceed $0.5$. When compared with the baseline experiments, \ie using the farthest point to the cluster center, the contrast becomes more obvious where the error is about $0.8$.

\begin{table}[t]
\centering
\caption{\textbf{Tightness results under varying number of linearization points $k$.} The table shows the tightness results using varying number of linearization points (\ie $k$-means cluster centers) averaged over all testing MNIST classes under different values of input noise $\sigma$. \newline*As for the Baseline experiment, the linearization points are set to be the farthest instances from the clusters' centers.}
\scalebox{0.95}{
\begin{tabular}{c|c|c|cc}
\toprule
$\sigma$ & $k$ & $\mathbb{E}[\mathbb{E}_{error}]$ & $\mathbb{E}[\text{var}_{error}]$ & $\mathbb{E}[\text{var}_{error}^{\text{new}}]$ \\
\midrule
0.001 & 250  &  $0.524 \pm 0.594$ &  $0.282 \pm 0.207$ &  $\mathbf{0.205 \pm 0.170}$ \\
&250\textsuperscript{*}&$0.846 \pm 0.688$&$\mathbf{0.293 \pm 0.220}$&$0.303 \pm 0.238$\\\cmidrule{2-5}
      & 500  &  $0.474 \pm 0.572$ &  $0.269 \pm 0.195$ &  $\mathbf{0.194 \pm 0.157}$ \\
&500\textsuperscript{*}&$0.766 \pm 0.684$&$0.277 \pm 0.206$&$\mathbf{0.276 \pm 0.221}$\\\cmidrule{2-5}
      & 1000 &  $0.414 \pm 0.541$ &  $0.260 \pm 0.191$ &  $\mathbf{0.185 \pm 0.152}$ \\
      & 2500 &  $0.315 \pm 0.481$ &  $0.241 \pm 0.175$ &  $\mathbf{0.165 \pm 0.134}$ \\
      & 5000 &  $0.194 \pm 0.387$ &  $0.224 \pm 0.163$ &  $\mathbf{0.140 \pm 0.115}$ \\
      & 10000 &  $0.003 \pm 0.027$ &  $0.200 \pm 0.142$ &  $\mathbf{0.100 \pm 0.078}$ \\
\midrule
0.010 & 250  &  $0.525 \pm 0.596$ &  $0.291 \pm 0.209$ &  $\mathbf{0.210 \pm 0.172}$ \\
&250\textsuperscript{*}&$0.859 \pm 0.694$&$\mathbf{0.284 \pm 0.210}$&$0.287 \pm 0.232$\\\cmidrule{2-5}
      & 500  &  $0.470 \pm 0.569$ &  $0.275 \pm 0.202$ &  $\mathbf{0.192 \pm 0.163}$ \\
&500\textsuperscript{*}&$0.763 \pm 0.678$&$0.275 \pm 0.203$&$\mathbf{0.259 \pm 0.217}$\\\cmidrule{2-5}
      & 1000 &  $0.411 \pm 0.537$ &  $0.263 \pm 0.190$ &  $\mathbf{0.174 \pm 0.145}$ \\
      & 2500 &  $0.313 \pm 0.478$ &  $0.243 \pm 0.176$ &  $\mathbf{0.143 \pm 0.125}$ \\
      & 5000 &  $0.198 \pm 0.388$ &  $0.227 \pm 0.163$ &  $\mathbf{0.101 \pm 0.101}$ \\
      & 10000 &  $0.023 \pm 0.103$ &  $0.203 \pm 0.143$ &  $\mathbf{0.025 \pm 0.020}$ \\
\midrule
0.050 & 250  &  $0.516 \pm 0.596$ &  $0.263 \pm 0.193$ &  $\mathbf{0.193 \pm 0.159}$ \\
&250\textsuperscript{*}&$0.848 \pm 0.689$&$\mathbf{0.285 \pm 0.214}$&$0.299 \pm 0.238$\\\cmidrule{2-5}
      & 500  &  $0.464 \pm 0.572$ &  $0.251 \pm 0.184$ &  $\mathbf{0.177 \pm 0.149}$ \\
&500\textsuperscript{*}&$0.761 \pm 0.678$&$\mathbf{0.268 \pm 0.202}$&$0.270 \pm 0.220$\\\cmidrule{2-5}
      & 1000 &  $0.405 \pm 0.538$ &  $0.242 \pm 0.181$ &  $\mathbf{0.163 \pm 0.142}$ \\
      & 2500 &  $0.312 \pm 0.475$ &  $0.226 \pm 0.166$ &  $\mathbf{0.138 \pm 0.120}$ \\
      & 5000 &  $0.216 \pm 0.388$ &  $0.211 \pm 0.155$ &  $\mathbf{0.110 \pm 0.098}$ \\
      & 10000 &  $0.109 \pm 0.235$ &  $0.192 \pm 0.139$ &  $\mathbf{0.070 \pm 0.049}$ \\
\midrule
0.100 & 250  &  $0.507 \pm 0.591$ &  $0.240 \pm 0.183$ &  $\mathbf{0.192 \pm 0.162}$ \\
&250\textsuperscript{*}&$0.823 \pm 0.688$&$\mathbf{0.298 \pm 0.220}$&$0.332 \pm 0.250$\\\cmidrule{2-5}
      & 500  &  $0.463 \pm 0.567$ &  $0.234 \pm 0.175$ &  $\mathbf{0.185 \pm 0.153}$ \\
&500\textsuperscript{*}&$0.751 \pm 0.671$&$\mathbf{0.290 \pm 0.219}$&$0.315 \pm 0.239$\\\cmidrule{2-5}
      & 1000 &  $0.412 \pm 0.535$ &  $0.226 \pm 0.169$ &  $\mathbf{0.177 \pm 0.147}$ \\
      & 2500 &  $0.337 \pm 0.477$ &  $0.218 \pm 0.163$ &  $\mathbf{0.165 \pm 0.132}$ \\
      & 5000 &  $0.268 \pm 0.406$ &  $0.208 \pm 0.156$ &  $\mathbf{0.158 \pm 0.116}$ \\
      & 10000 &  $0.205 \pm 0.313$ &  $0.196 \pm 0.146$ &  $\mathbf{0.153 \pm 0.082}$ \\
\bottomrule
\end{tabular}
}
\label{linearization_sensitivity}
\end{table}

\subsection{Noise Construction}
After establishing the tightness of our expressions compared to Monte Carlo estimates, we show more practical applications of these expressions, in which the output mean and variance expressions can be used to construct noise with certain properties. In particular, we are interested in showing that samples from a carefully crafted Gaussian distribution can act as an adversary. This goes against the common belief that Gaussian noise is too simple for such a task. In this section, we show insightful results on how to construct targeted and non-targeted Gaussian adversarial attacks. Moreover, we also show that such expressions can be leveraged to construct sparse and smooth Gaussian adversarial attacks that are perceptually feasible. It is to be noted here that this section is concerned about establishing the fact that Gaussian noise can act as an adversary while being perceptually feasible and not to particularly achieve state-of-the-art results on the task of adversarial attacks. The problem setup is as follows: given an image $\mathbf{M}$, whose predicted class is $i$, the task is to add noise $\mathbf{x} \sim \mathcal{N}(\mu_x,\Sigma_x)$ to $\mathbf{M}$ such that the \textit{expected} prediction score of the network of $\mathbf{M} + \mathbf{x}$ is $j \neq i$. If such noise exists, we say the network is fooled in expectation. To keep the optimization and the number of variables manageable, we only consider the case of isotropic Gaussian distributions, \ie $\Sigma_x = \sigma^2 \mathbf{I}_n$. We define $\mathcal{E}_i^\mathbf{M}(\mu_x, \sigma^2) \equiv \mathbb{E}[\mathbf{g}_i(\mathbf{M}+\mathbf{x}_{(\mu_x,\sigma^2\mathbf{I}_n)})]$ to avoid text clutter. In the following experiments, the two-stage linearization is performed around $\mathbf{M}$ at $l=3$ for LeNet and $l=7$ for AlexNet.

\vspace{4pt}\noindent \textbf{Targeted Attacks.~} On the MNIST dataset, we specify a target class $j$ and we construct a noise that can fool LeNet in expectation by solving the following optimization:
\begin{equation}
    \begin{aligned}
        \underset{\mu_x,\sigma}{\arg\min}& \left (\max_{k \neq j}\left(\mathcal{E}^\mathbf{M}_k(\mu_x, \sigma^2 )\right) - \mathcal{E}^\mathbf{M}_j(\mu_x, \sigma^2 )\right) \\
        \text{s.t.~~}& 0 < \sigma^2 \leq 2, ~~~ -\beta \mathbf{1}_n \leq \mu_x \leq \beta \mathbf{1}_n. 
    \end{aligned}
    \label{opt1}
\end{equation}

\noindent Note that for any pair $(\mu_x,\sigma)$ for which the previous objective is negative, the largest expected prediction among all classes occurs at the target class $j$. In this experiment, we set $\beta = 30$ and solve problem \eqref{opt1} with an interior-point solver. Note that the range of pixel values of MNIST images is $[-127.5,127.5]$. Figure \ref{trageted_noise} shows examples of noisy versions of an image from class $9$ that fool LeNet in expectation with multiple target classes (\ie $j\in\{2,3,4,7,8\}$). Not every target class is easily targeted with small $\beta$ because of the distance in their prediction scores. We verify that the constructed noise actually fools the network by sampling 10 samples from the learned distribution, passing each noisy input through LeNet, and verifying that at least $90\%$ of the predicted class flips are from $9$ to the target class $j$.

\begin{figure}[t]
\begin{center}
	 	\includegraphics[width = 0.48\textwidth]{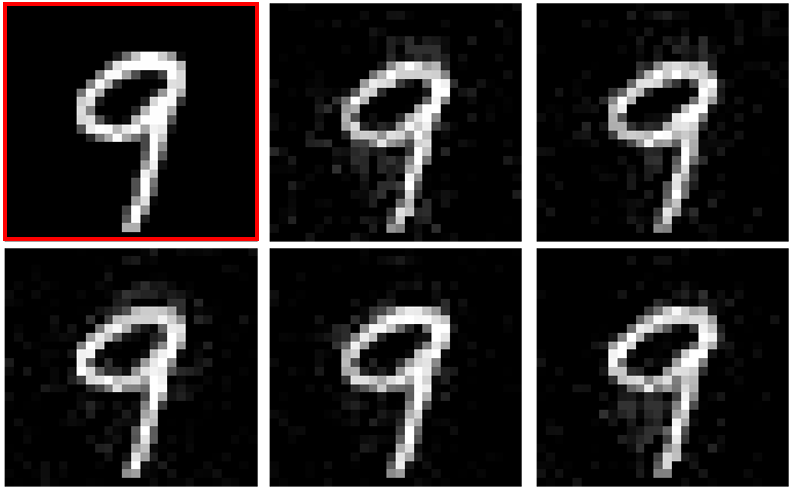}
	 	\caption{\textbf{Targeted attacks.} The figure shows noisy images that fool LeNet. The images from top-left to bottom-right are the original image from MNIST and the noisy versions classified as $2$, $3$, $4$, $7$, and $8$, respectively.}
\label{trageted_noise}
\end{center}
\end{figure}

\vspace{4pt}\noindent \textbf{Non-Targeted Attacks with $\alpha\%$-Pixel Support.~} Inspired by the findings of some recent work \cite{su2017one}, we demonstrate that we can construct additive noise that only corrupts $\alpha\%$ of the pixels in an input image, but still changes the class prediction. Here, we use LeNet on MNIST and AlexNet on ImageNet. In this case, we do not specify the target class $j$ but rather we optimize for the prediction scores of the correct class $i$ to be less than the maximum prediction score. The underlying optimization is formulated as follows:
\begin{equation}
    \begin{aligned}
        \underset{\mu_x^\alpha, \sigma}{\arg\min}& \left (\mathcal{E}^\mathbf{M}_i(\mu_x^\alpha, \sigma^2 ) - \max_{k \neq i}\left(\mathcal{E}^\mathbf{M}_k(\mu_x^\alpha, \sigma^2)\right) \right)\\
        \text{s.t.~~}& 0 < \sigma^2 \leq 2, ~~~ -\beta \mathbf{1}_{\alpha n} \leq \mu_x^\alpha \leq \beta \mathbf{1}_{\alpha n}.
    \end{aligned}
    \label{opt2}
\end{equation}

\noindent The optimization variable $\mu_x^\alpha$ indicates the set of sparse pixels ($\alpha\%$ of the total number of pixels) in $\mu_x$ that will be corrupted, while the rest of pixels are set to $0$. The locations of the corrupted pixels are randomly chosen and fixed before solving the optimization. Two experiments are conducted on few images, one on MNIST and the other on ImageNet. Figures \ref{mnist_flip} and \ref{imagenet_flip} show examples of noisy images constructed by solving Equation \eqref{opt2} with $\alpha=4\%$ to fool LeNet on MNIST and  $\alpha=2\%$  to fool AlexNet on ImageNet. Since there are much fewer pixels to flip the prediction of the network and similar to the single pixel attack in \cite{su2017one}, we increase the permissible range of mean noise by setting $\beta = 550$ for MNIST and $\beta = 75$ for ImageNet. Note, in these experiments, we assume that $\mathbf{M} \in [0,255]^n$.

\begin{figure}[t]
 	\includegraphics[width=0.48\textwidth]{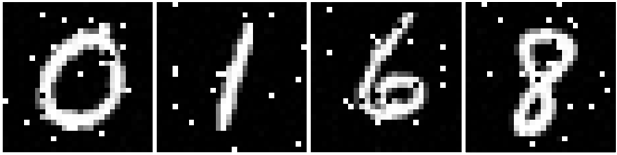}
 	\caption{\textbf{Non-targeted attacks with $\alpha \%$-pixel support on MNIST.~} The noisy digit 9 as predicted by LeNet as  $2$, $4$, $2$, and $2$ (from left to right), after adding noise generated by Equation \eqref{opt2}. The first image marked in red is the noise free sample.}
\label{mnist_flip}
\end{figure}

\begin{figure}[t]
 	\includegraphics[width=0.48\textwidth]{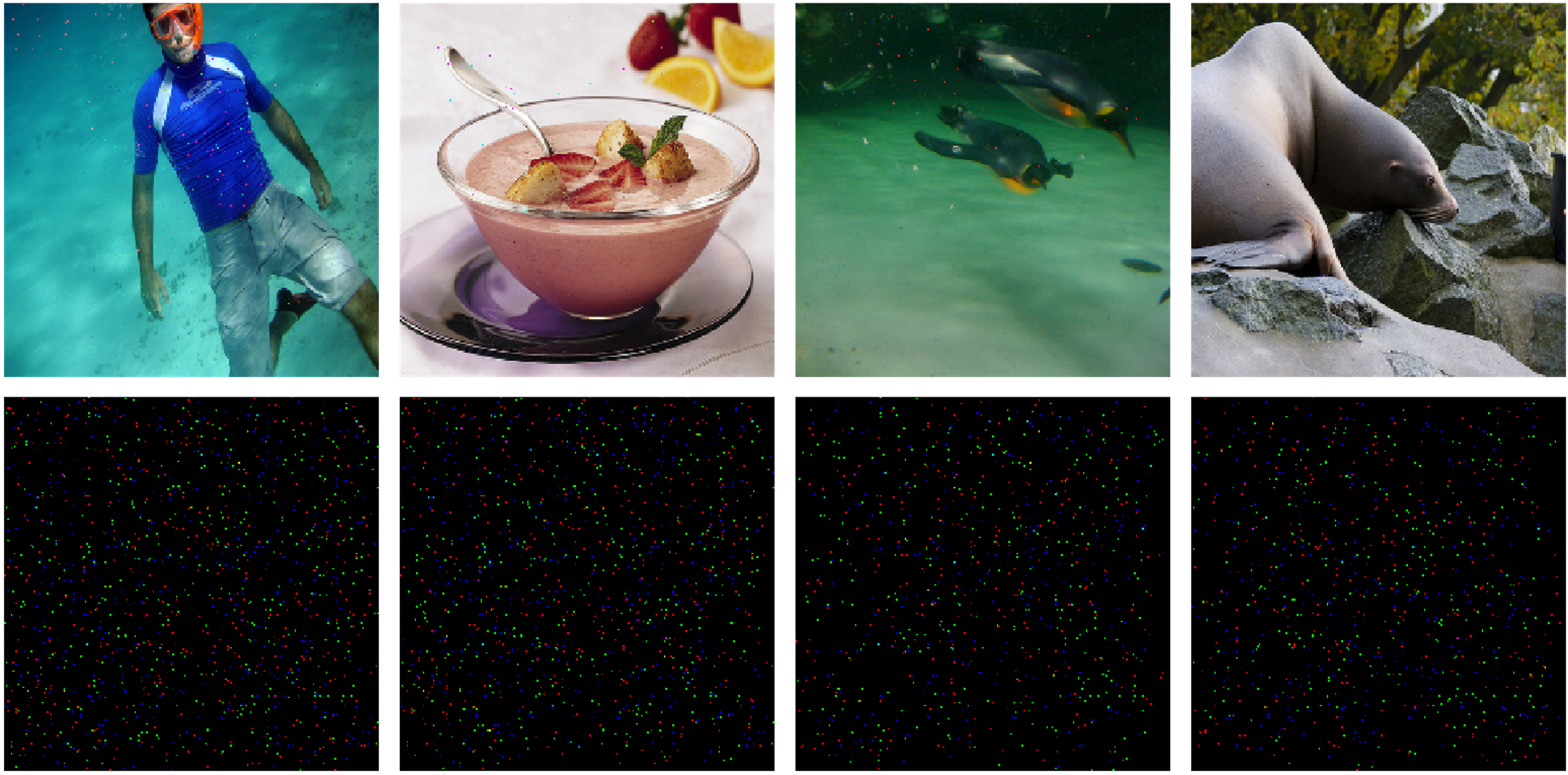}
 	\caption{\textbf{Non-targeted Attacks with $\alpha \%$-pixel support on ImageNet.~} The figure in the top row the noisy images that fool AlexNet as generated by solving Equation \eqref{opt2}. The second row shows the generated noise with $\alpha = 2\%$ of the total pixels in the image.}
\label{imagenet_flip}
\end{figure}

\vspace{4pt}\noindent \textbf{Non-Targeted Attacks with Sparse and Smooth Pixel Support.~} In the previous section, we optimized over a randomly selected support, $\alpha \%$, which may not hold any structure nor is perceptually feasible. To that end, and towards constructing more meaningful structured noise that is more perceptually feasible, we instead optimize over the complete pixel support while enforcing both sparsity and smoothness, in this subsection. In other words, we are interested in designing a Gaussian distribution with, for ease, an identity input covariance ($\Sigma_x = \mathbf{I}_n$) but with mean $\mu_x$ that is both sparse and smooth, such that samples from this distribution cause the network to incorrectly predict the class of the input. The optimization of interest can be formulated as follows:
\begin{equation}
    \label{opt:sparse_plus_smooth}
    \begin{aligned}
     \underset{\mu_x}{\arg\min} & ~~~~~ \|\mu_x\|_1 +  \Omega(\nabla_x \mu_x, \nabla_y \mu_x) ~~ \\
     \text{s.t.}~~~ & \mathcal{E}^\mathbf{M}_i( \mu_x, 1) - \max_{k \neq i}\left(\mathcal{E}^\mathbf{M}_k(\mu_x,1)\right) \leq - \gamma,
    \end{aligned}
\end{equation}
where $\gamma>0$ is some constant controlling the misclassification confidence, $\Omega(x,y) = \sum_i \sqrt{x_i^2 + y_i^2}$ and $\nabla_x, \nabla_y \in \mathbb{R}^{n \times n}$ are gradient operators in the x and y directions, respectively. The $\ell_1$ norm in the objective is commonly used to encourage sparsity, while $\Omega(\nabla_x \mu_x, \nabla_y \mu_x)$ is well-known to be the total variation regularizer, which is commonly used in various low-level image processing tasks (\eg denoising, deblurring, etc.), to encourage smoothness via sparse spatial gradients. We set $\gamma=0$ in all experiments unless stated otherwise. We conduct two sets of experiments one with LeNet on the MNIST dataset and another set with a variant of AlexNet on the Facial Emotion Recognition dataset from Kaggle. Note that the Emotion dataset consists of $35887$ frontal images of faces depicting 7 emotions, where the best performing network on Kaggle achieves a classification accuracy of 68\%. Since the Kaggle models are not publicly available, we choose to train a variant of AlexNet to account for the difference in the input resolution between ImageNet and Emotion dataset images. This AlexNet variant achieves a comparable test accuracy of 64\%. We solve \eqref{opt:sparse_plus_smooth} on both datasets (LeNet on MNIST and AlexNet on the Emotion dataset) with $l=3$, where the points of linearization are random images from the test set that are classified correctly by the network. Figure \ref{fig:non_l1tv} shows few qualitative adversarial examples from the MNIST dataset. Since the solution to \eqref{opt:sparse_plus_smooth}, $\mu_x$, can be both positive and negative, we visualize both quantities in the third and forth columns in Figure \ref{fig:non_l1tv}. One can observe how the sparse noise is both smooth and structured due to the proposed total variation regularizer. For instance, the noise that corresponds to misclassifying the digit 3 to 9 in Figure \ref{fig:non_l1tv} is indeed located at the top left part of the digit perceptually altering digit 3 into looking more like digit 9. This confirms that the proposed optimization \eqref{opt:sparse_plus_smooth} indeed results in a Gaussian distribution, where noise sampled from that distribution is more perceptually feasible for the task of network misclassification. To perform targeted attacks (with target class $j$), one can simply replace the constraint in \eqref{opt:sparse_plus_smooth} with:
\begin{equation}
    \label{eq:targeted_constraint}
    \begin{aligned}
    \max_{k \neq j}\left(\mathcal{E}^\mathbf{M}_k(\beta\mu_x, 1)\right) - \mathcal{E}^\mathbf{M}_j(\beta\mu_x, 1) \leq - \gamma.
    \end{aligned}
\end{equation}

\begin{figure*}[t]
    \centering
    \includegraphics{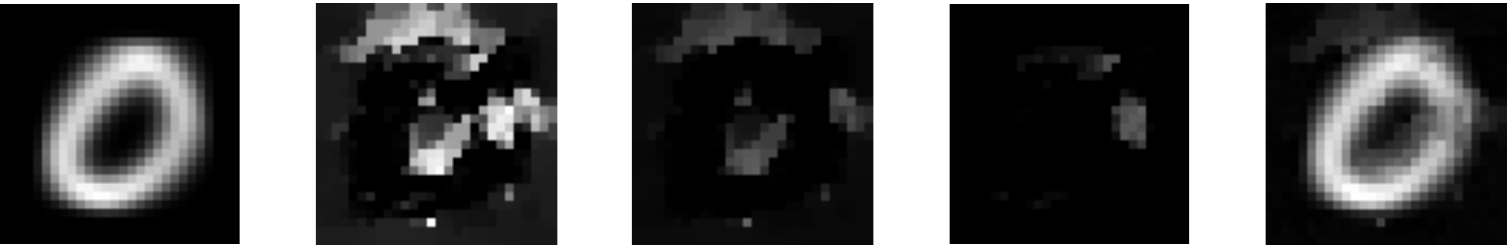}
    \includegraphics{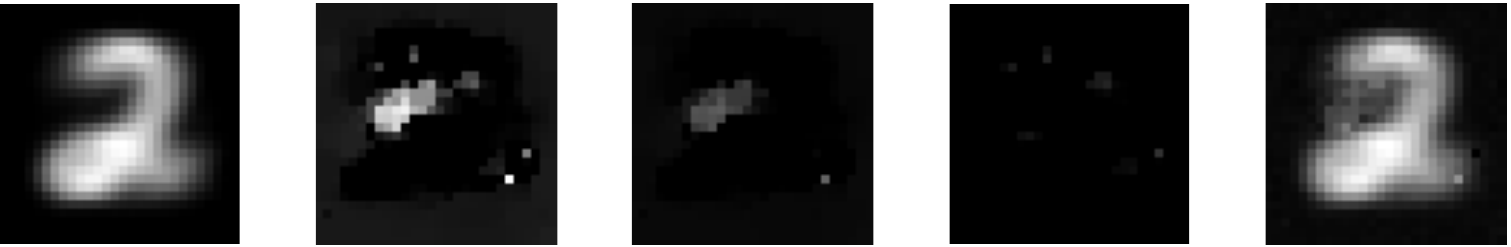}
    \includegraphics{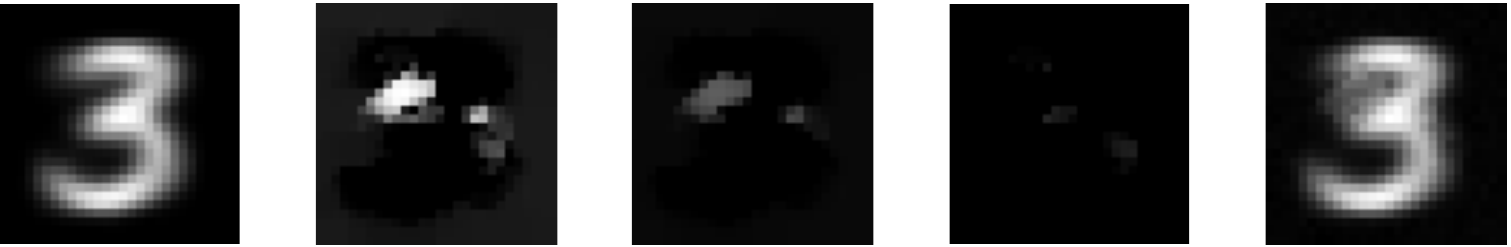}
    \includegraphics{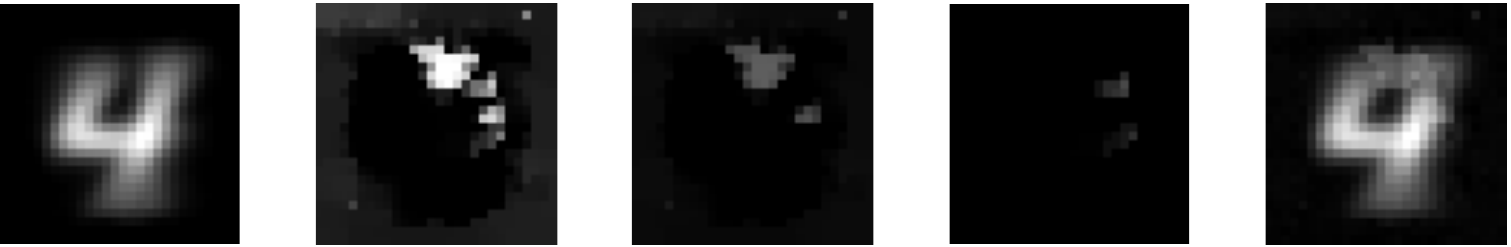}
    \includegraphics{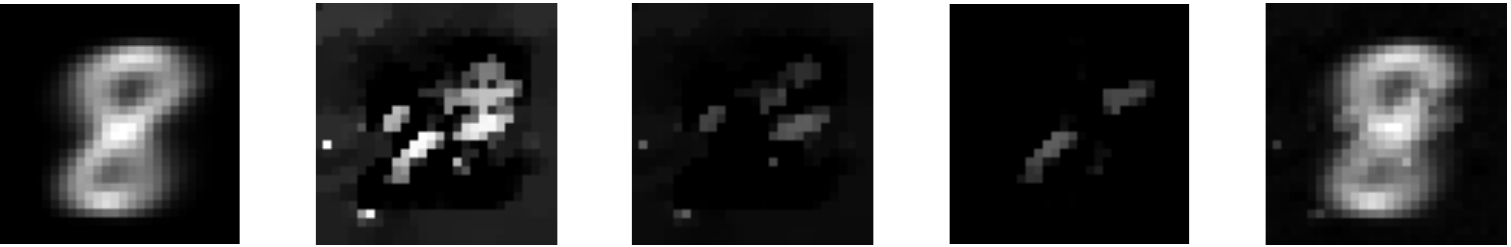}
    \caption{\textbf{Non-targeted attacks with sparse and smooth pixel support.~} We show that the solution $\mu_x$ to Problem \eqref{opt:sparse_plus_smooth} is indeed both sparse and smooth. The columns from left to right are the mean image for a given class, the absolute value of $\mu_x$, the decomposition into positive and negative values, respectively and the last column shows the first column added with noise sampled from $\mathcal{N}(\mu_x,\mathbf{I})$. After the addition of the structured noise, the network classified the noisy images as $6$, $8$, $9$, $9$, and $9$, respectively.}
    \label{fig:non_l1tv}
\end{figure*}

\begin{figure*}
    \centering
    \includegraphics{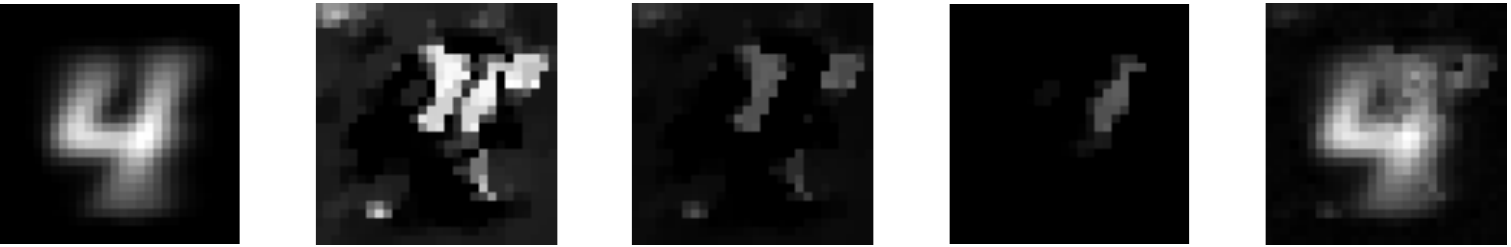}
    \includegraphics{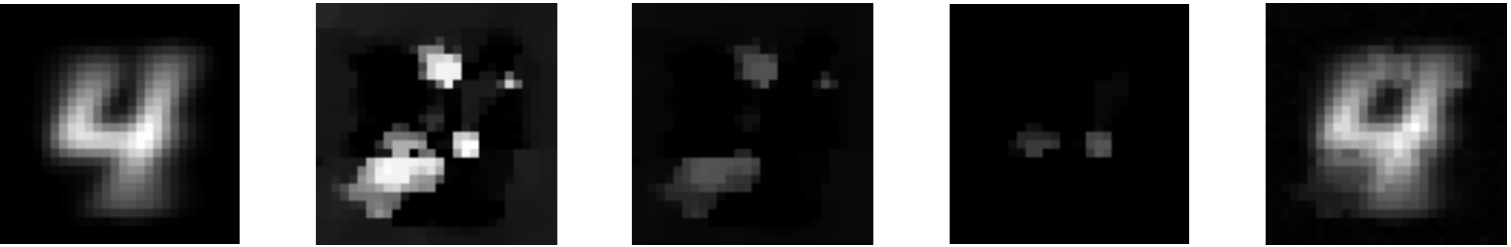}
    \includegraphics{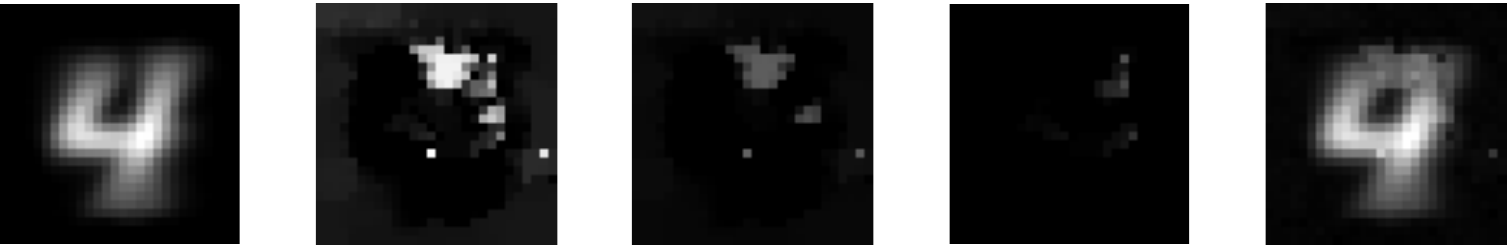}
    \caption{\textbf{Targeted attacks with sparse and smooth pixel support.~} The columns from left to right are the mean image for a given class, the absolute value of $\mu_x$, the decomposition into positive and negative values, respectively and the last column shows the first column added with noise sampled from $\mathcal{N}(\mu_x,\mathbf{I})$. After the addition of the structured noise, the network classified the noisy images as $5$, $8$, and $9$, respectively.}
    \label{fig:tar_l1tv}
\end{figure*}

\noindent Note that obtaining a feasible solution to problem \eqref{opt:sparse_plus_smooth} with the constraint replaced by \eqref{eq:targeted_constraint} constructs a Gaussian distribution, where samples from this distribution result in an expected prediction at class $j$ be higher than all other classes. We show in Figure \ref{fig:tar_l1tv} examples for targeted sparse smooth attacks that fool the network into classifying the digit 4 to the targeted classes 5, 8 and 9. In particular, it is interesting to observe that when the target class is digit 5 (first row of Figure \ref{fig:tar_l1tv}), the Gaussian sampled noise has negative and positive components to the top right and top parts of digit 4, respectively, which tend to be more perceptually feasible than the unconstrained noise case. A similar observation can be made for the cases when the target class is digit 8 or 9. For a more detailed quantitative experiment on the effectiveness of the proposed optimization in constructing targeted attacks, we conduct experiments on 70 randomly selected images from the Emotion dataset (10 per class). In particular, we construct targeted attacks from every class in the Emotion dataset to every other class. Table \ref{table:emotion_l1tv} reports the confusion matrix of the misclassification rate for every source-target pair. Note that with the proposed objective in \eqref{opt:sparse_plus_smooth} with the constraint replaced by \eqref{eq:targeted_constraint}, high misclassification rates to the targeted classes are effectively achieved with perceptually feasible attacks. This is consistent in all source-target pairs, as summarized in Table \ref{table:emotion_l1tv}. We show several qualitative results of these attacks in the \textbf{Appendix}.

\begin{table}[t]
\caption{\textbf{Targeted sparse Gaussian attacks.} The table shows the misclassification rate from every class to every other class by solving \eqref{opt:sparse_plus_smooth} with the targeted attack constraint in \eqref{eq:targeted_constraint} on 70 randomly selected images from the Emotion dataset (10 per class).}
\begin{center}
\begin{tabular}{c||c|c|c|c|c|c|c|c}
\toprule
\multicolumn{1}{c}{} & \multicolumn{1}{c}{} & \multicolumn{7}{c}{Target Class} \\
\midrule
\midrule
 & \multicolumn{1}{c|}{} & \textbf{0} & \textbf{1} & \textbf{2} & \textbf{3} & \textbf{4} & \textbf{5} & \textbf{6} \\ \cline{2-9}
\parbox[t]{2mm}{\multirow{7}{*}{\rotatebox[origin=c]{90}{Source Class}}}
 & \textbf{0} & - & 1 & 0.7 & 1 & 0.7 & 0.9 & 1 \\ \cline{2-9} 
 & \textbf{1} & 0.8 & - & 1 & 1 & 0.7 & 0.9 & 1 \\ \cline{2-9} 
 & \textbf{2} & 1 & 1 & - & 1 & 0.9 & 1 & 1 \\ \cline{2-9} 
 & \textbf{3} & 0.9 & 1 & 0.7 & - & 0.7 & 0.7 & 0.9 \\ \cline{2-9} 
 & \textbf{4} & 1 & 1 & 1 & 1 & - & 1 & 0.9 \\ \cline{2-9} 
 & \textbf{5} & 0.9 & 1 & 0.9 & 1 & 0.6 & - & 1 \\ \cline{2-9} 
 & \textbf{6} & 0.9 & 1 & 0.8 & 1 & 1 & 0.9 & - \\ 
 \bottomrule
\end{tabular}
\end{center}
\label{table:emotion_l1tv}
\end{table}

\vspace{4pt}\noindent \textbf{Misclassification Rate with Varying $\gamma$.~} Here, we explore the effect of varying the parameter $\gamma$, which controls the misclassification rate confidence of the Gaussian sampled attacks. Larger $\gamma$ should result in samples from the designed Gaussian with a larger expected prediction output for the target class $j$ than all other classes. Since the highest misclassification, as per Table \ref{table:emotion_l1tv}, occurred when the source-target pair is classes 5 and 4, respectively, we solve Problem \eqref{opt:sparse_plus_smooth} with class 4 as the targeted class in constraint \eqref{eq:targeted_constraint} for images from class 5. We randomly select 30 images from class 5 while varying $|\gamma|$. In Figure \ref{fig:varying_alpha}, we plot the average misclassification rate over the 30 samples, where we show on the x-axis the variation in the expected separation in network predictions, \ie $\max_{k \neq j}\left(\mathcal{E}^\mathbf{M}_k(\beta\mu_x, \mathbf{I})\right) - \mathcal{E}^\mathbf{M}_j(\beta\mu_x, \mathbf{I})$. As anticipated, we observe that the misclassification rate increases with the increase in separation between the two expected predictions. It is essential to note that, while larger $\gamma$ indeed results in a higher misclassification rate, it comes at the expense of solving a harder optimization problem.  

\vspace{4pt}\noindent \textbf{Sparsity with Varying $\gamma$.~} In addition to the role of $\gamma$ in increasing the misclassification rate, we study its impact on the sparsity of the sampled noise. This is essential towards understanding whether the attacks are resulting in a higher misclassification rates with larger $\gamma$ due to an increase in the noise support, \ie lower sparsity. Similar to the previous experiment, we plot the sparsity in the noise as a function of the expected separation in network predictions, \ie $\max_{k \neq j}\left(\mathcal{E}^\mathbf{M}_k(\beta\mu_x, \mathbf{I})\right) - \mathcal{E}^\mathbf{M}_j(\beta\mu_x, \mathbf{I})$. To measure sparsity, all values in the constructed noise that are $\leq 10^{-3}$ are considered to be zero. As shown in Figure \ref{fig:varying_alpha}, the sparsity starts at around $90\%$ and decreases only marginally. This indicates that there exist powerful Gaussian attacks that result in high misclassification rates, due to the large separation in the expected predictions, which are also effectively sparse.

\begin{figure}[t]
\centering
\includegraphics[width=0.4\textwidth]{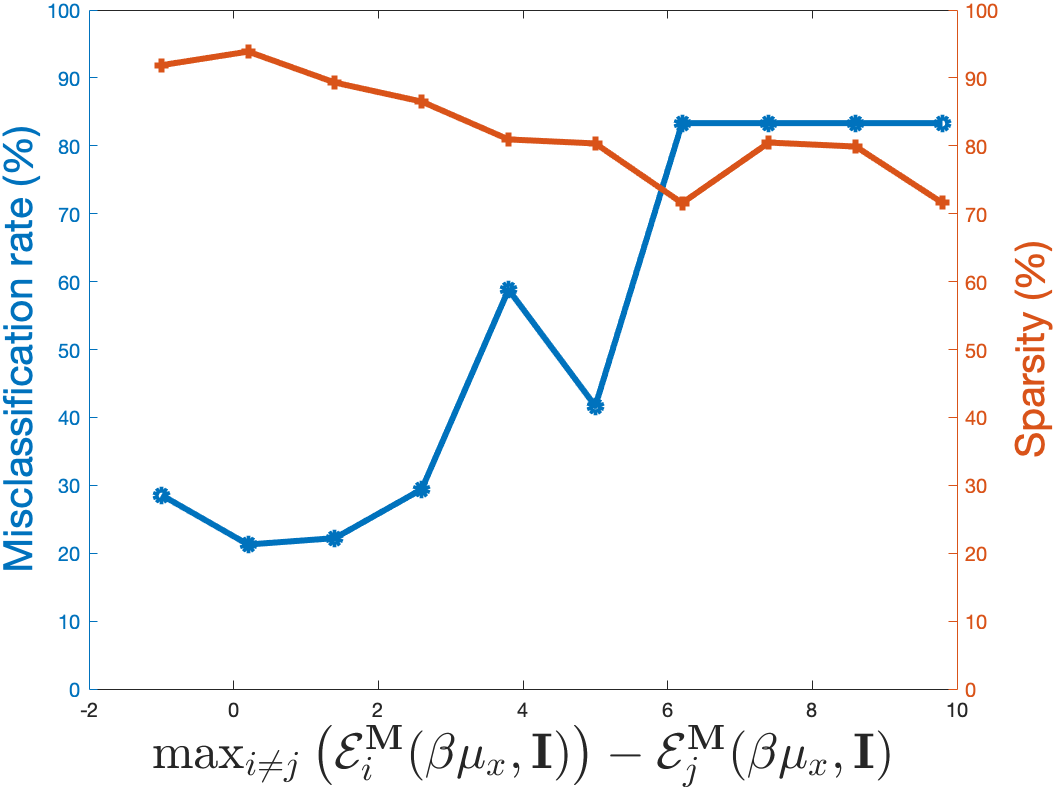}
\caption{\textbf{Misclassification and sparsity rates under varying separation of expected predictions.~} With the increases in the separation between the expected output predictions of the correct and incorrect class increases, the misclassification rate increases while maintaining the same sparsity in the Gaussian attacks.}
\label{fig:varying_alpha}
\end{figure}

\section{Conclusion}
In this paper, we extend and generalize our preliminary results in \cite{bibi2018analytic}, where we derive the exact second moment expression for a small PL neural network in the form (Affine, ReLU, Affine) under no assumptions on the input mean of the Gaussian distribution. We conduct extensive experiments demonstrating the role of the layer, $l$, at which two-stage linearization is performed, and over strategies in selecting the point of network linearization. Moreover, we demonstrate that such expressions can be used to construct targeted and non-targeted Gaussian adversarial attacks that are both sparse and smooth (more perceptually feasible).

%%%%%%%%%%%%%%%%%%%%%%%%%%%%%%% Appendices %%%%%%%%%%%%%%%%%%%%%%%%%%%%%%%%%%%%%%%%%%%%%%
% if have a single appendix:
%\appendix[Proof of the Zonklar Equations]
% or
%\appendix  % for no appendix heading
% do not use \section anymore after \appendix, only \section*
% is possibly needed

% use appendices with more than one appendix
% then use \section to start each appendix
% you must declare a \section before using any
% \subsection or using \label (\appendices by itself
% starts a section numbered zero.)
%

% \appendices
% \section{Proof of the First Zonklar Equation}
% Appendix one text goes here.

% % you can choose not to have a title for an appendix
% % if you want by leaving the argument blank
% \section{}
% Appendix two text goes here.
\appendix
% \onecolumn

\section*{Qualitative Results on the Emotion Dataset}
The classes in the Emotion dataset are ``Anger'', ``Disgust'', ``Fear'', ``Happiness'', ``Sadness'', ``Surprise'' and ``Neutral''. In a similar fashion to the MNIST experiments in Figure \ref{fig:non_l1tv}, we visualize in Figure \ref{fig:non_l1tv_emotion} the structured noise upon solving the total variation regularized Problem \eqref{opt:sparse_plus_smooth}. Despite that encoding high-level semantic information, such as emotions, in pixel intensity space is generally very difficult, some of the presented qualitative results in Figure \ref{fig:non_l1tv_emotion} are perceptually aligned. For instance, in the last row, the Gaussian noise is structured around the eyebrows resulting in misclassifying ``Neutral'' as ``Sadness''. A similar observation is to be noted for the first row where the ``Happiness'' image is misclassified as ``Sadness'' where all the noise is structured around the eyebrows and the chin.

\begin{figure*}[!ht]
    %% 0 - Anger
    %% 1 - Disgust
    %% 2 - Fear
    %% 3 - Happiness
    %% 4 - Sadness
    %% 5 - Surprise
    %% 6 - Neutral
    \centering
    \includegraphics[scale=0.39]{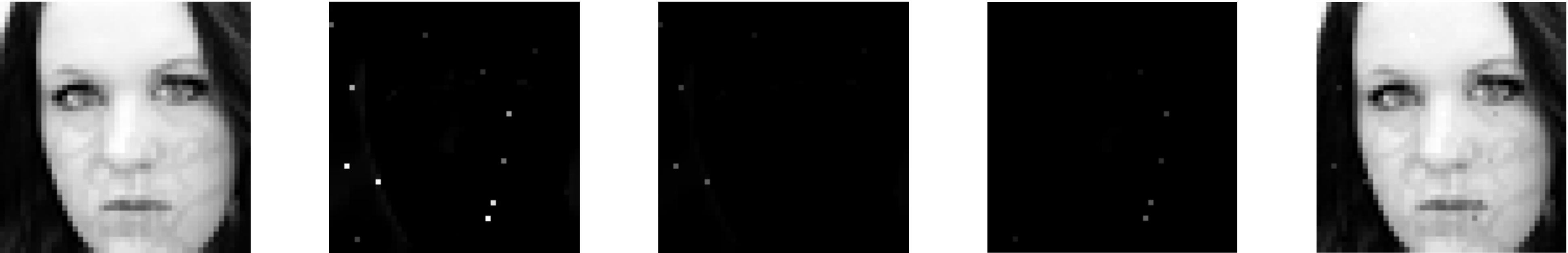}
    \includegraphics[scale=0.39]{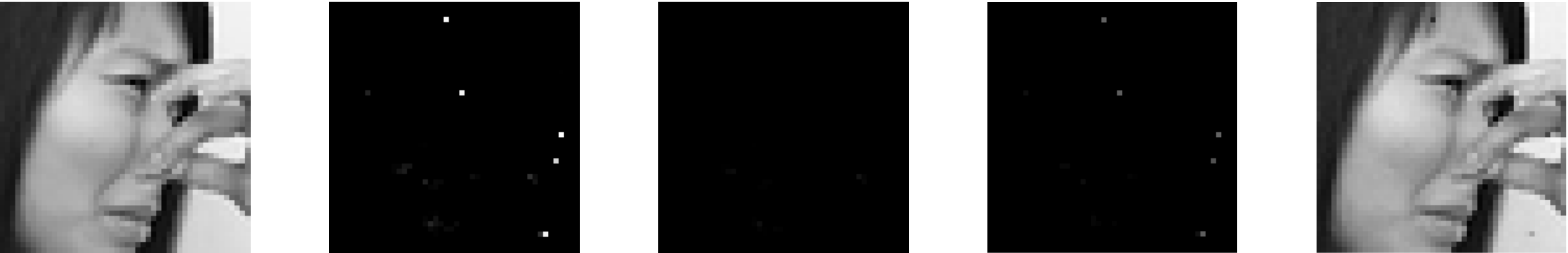}
    \includegraphics[scale=0.39]{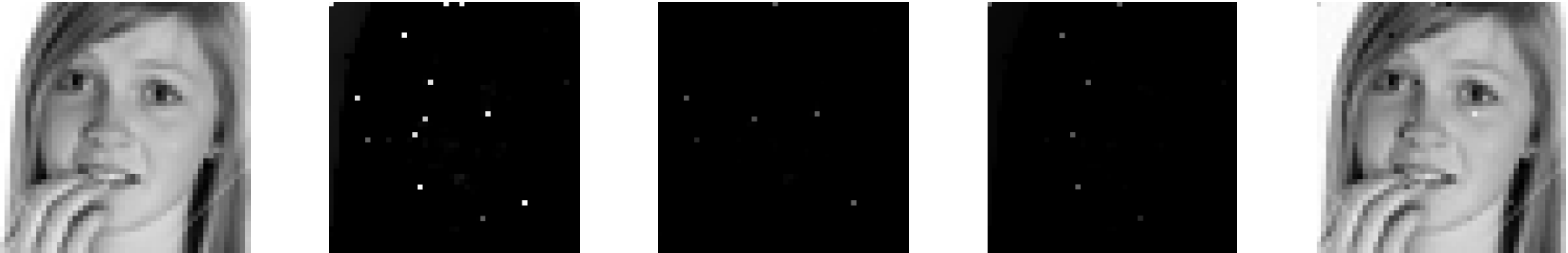}
    \includegraphics[scale=0.39]{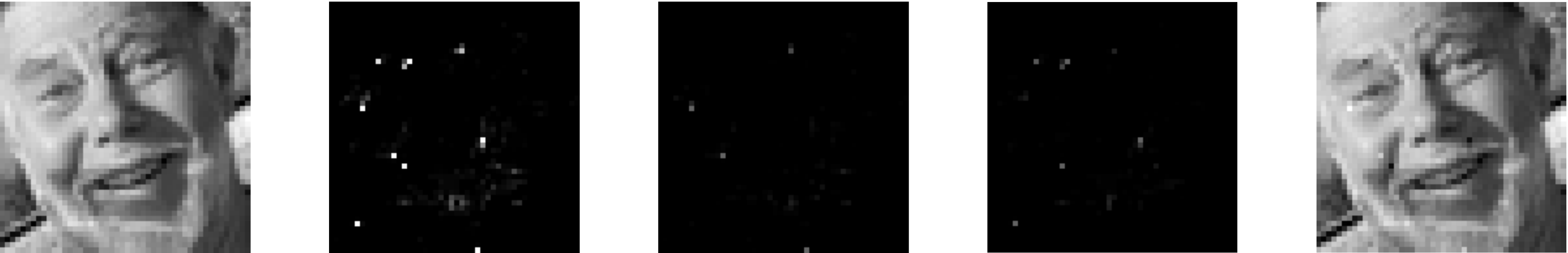}
    \includegraphics[scale=0.39]{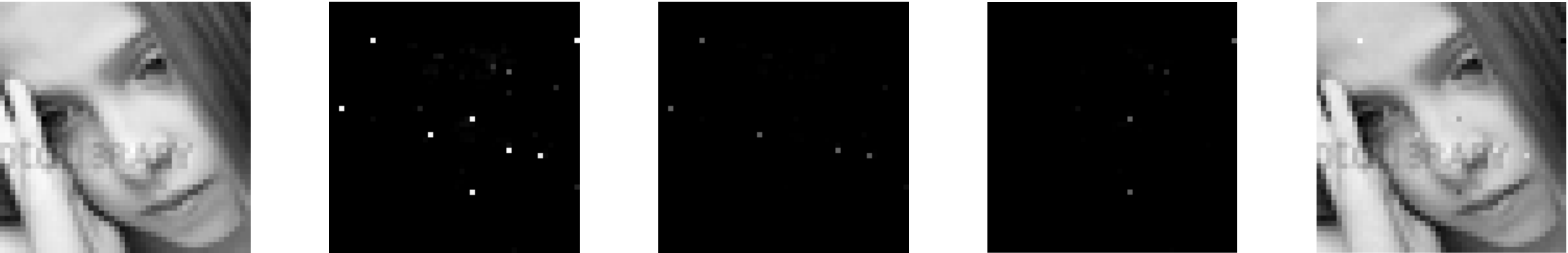}
    \includegraphics[scale=0.39]{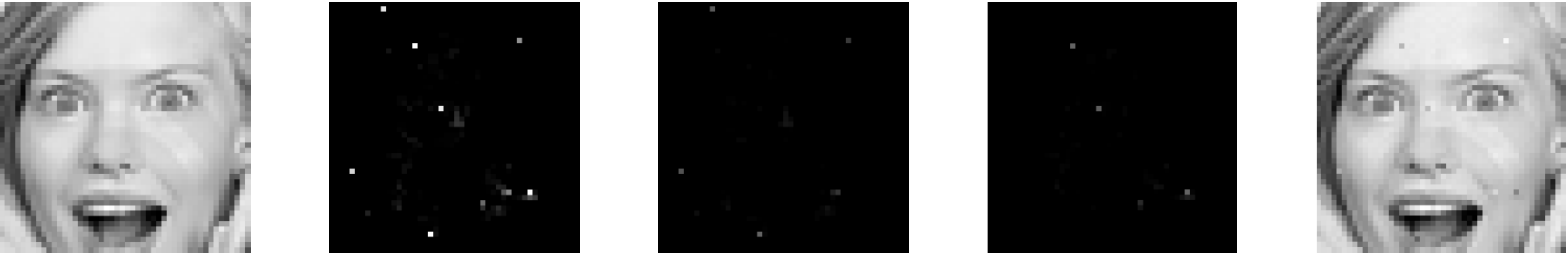}
    \includegraphics[scale=0.39]{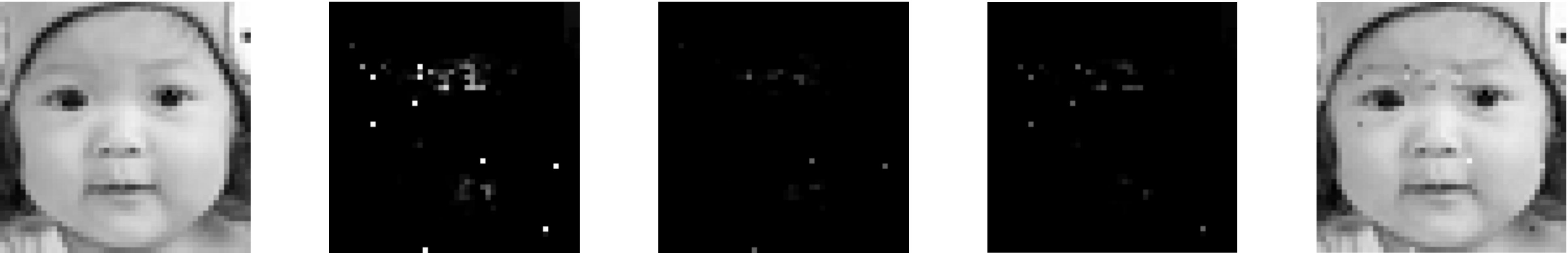}
    \caption{\textbf{Non-targeted attacks with sparse and smooth pixel support.~} We show that the solution $\mu_x$ to Problem \eqref{opt:sparse_plus_smooth} is indeed both sparse and smooth and in some cases perceptually aligned on the Emotion dataset. The columns from left to right are an image from the Emotion dataset to be optimized over from a given class, the absolute value of $\mu_x$, the decomposition into positive and negative values, respectively and the last column shows the first column added with noise sampled from $\mathcal{N}(\mu_x,\mathbf{I})$. After the addition of the structured noise, the network misclassified the images as as Neutral, Fear, Neutral, Sadness, Neutral, Fear and Sadness from Anger, Disgust, Fear, Happiness, Sadness, Surprise, and Neutral, respectively.}
    \label{fig:non_l1tv_emotion}
\end{figure*}

\section*{Proof of Theorem \ref{theo1}}
\begin{theorem_supp}
For any function in the form of $\mathbf{g}(\mathbf{x})$ where $\mathbf{x} \sim \mathcal{N} \left(\mu_x, \Sigma_x \right)$, we have:
\begin{equation}
\begin{aligned}
\mathbb{E}[\mathbf{g}_i(\mathbf{x})] =  \sum_{v=1}^p &\mathbf{B}(i,v)\left(  \frac{1}{2}  \bar{\mu}_{v}  -  \frac{1}{2}  \bar{\mu}_{v}  \text{erf}\left(\frac{-\bar{\mu}_{v}}{\sqrt{2} \bar{\sigma}_{v}}\right) \right. \\
& \left. +  \frac{1}{\sqrt{2 \pi}}\bar{\sigma}_{v} \exp\left(-\frac{\bar{\mu}_{v}^2}{2 \bar{\sigma}_{v}^2} \right) \right) + \mathbf{c}_2(i) 
\notag
\end{aligned}
\end{equation}
where $\bar{\mu}_{v} = \left(\mathbf{A} \mu_x + \mathbf{c}_1\right)(v)$,  $\bar{\Sigma} = \mathbf{A} \Sigma_x \mathbf{A}^\top$, $\bar{\sigma}_{v}^2 = \bar{\Sigma}(v,v)$ and $\text{erf}\left(x\right) = \frac{2}{\sqrt{\pi}} \int_0^x \text{e}^{-t^2}dt$ is the error function.
\end{theorem_supp}

\begin{proof} Based on Remark (\ref{remark1}) and noting that $\left(\mathbf{A}\mathbf{x} + \mathbf{c}_1\right) \sim \mathcal{N} \left(\bar{\mu},\bar{\Sigma}\right)$, we have:
\begin{equation}
\begin{aligned}
&\tilde{\mu}(i) =  \int_{0}^{\infty} z \frac{1}{\sqrt{2 \pi } \bar{\sigma}_i} \exp\left(-\frac{(z - \bar{\mu}_i)^2}{2\bar{\sigma}_i^2} \right)dz  \\
& =  \frac{1}{2}  \bar{\mu}_i   - \frac{1}{2} \bar{\mu}_i \text{erf}\left(\frac{-\bar{\mu}_i}{\sqrt{2} \bar{\sigma}_i} \right)+ \frac{1}{\sqrt{2 \pi}}\bar{\sigma}_i \exp\left(-\frac{\bar{\mu}_i^2}{2 \bar{\sigma}_i^2} \right)
\notag
\end{aligned}
\end{equation}
Thus, $\mathbb{E}[\mathbf{g}_i(\mathbf{x})] = \sum_{v=1}^p\mathbf{B}(i,v) \tilde{\mu}(v)+ \mathbf{c}_2(i)$.
\end{proof}

%%%%%%%%%%%%%%%%%%%%%%%%%%%%%%%%%%%%%%%%%%%%%

\section*{Proof of Lemma \ref{lemma1}}
\begin{lemma_supp}
The PDF of $q^2(x) = \max^2(x,0) : \mathbb{R} \rightarrow \mathbb{R}$ where $x \sim \mathcal{N}\left(0,\sigma_x^2 \right)$ is :
\begin{equation}
\begin{aligned}
f_{q^2}(x) =  \frac{1}{2} \delta(x) + \frac{1}{2\sqrt{x}} f_x(\sqrt{x}) u(\sqrt{x})
\notag
\end{aligned}
\end{equation}
and its first moment is $\mathbb{E}[q^2(x)] = \frac{\sigma_x^2}{2}$.
\end{lemma_supp}

\begin{proof}
By the cumulative distribution function (CDF):
\begin{align} 
\label{cdfunc}
F_{q^2}(c) &= \mathbb{P}(\text{max}^2(x,0) \leq c)%\\
= \frac{1}{2}\delta(c) + F_x(\sqrt{c}) u\left(\sqrt{c}\right)
\end{align}
$F_{q^2}$ and $F_x$ are the CDFs of $q^2$ and $x$. Differentiating the smooth part of (\ref{cdfunc}) with respect to $c$ completes the proof.
\end{proof}

%%%%%%%%%%%%%%%%%%%%%%%%%%%%%%%%%%%%%%%%%%%%%

\section*{Proof of Lemma \ref{price_exten}}
\begin{lemma_supp}
Let $\mathbf{x} \in \mathbb{R}^{n} \sim \mathcal{N}(\mu_x,\Sigma_x)$ for any even p, where $\sigma_{ij} = ~~ \Sigma_x(i,j) ~\forall i\neq j$. Under mild assumptions on the nonlinear map $\Psi :
\mathbb{R}^{n} \rightarrow \mathbb{R}$, we have $\frac{\partial^{\frac{p}{2}} \mathbb{E}[\Psi(\mathbf{x})]}{\prod_{\forall \text{odd} i}\partial \sigma_{ii+1}}$ $= \mathbb{E} [\frac{\partial^p \Psi(\mathbf{x})}{\partial x_1 \dots \partial
x_p} ]$.
\end{lemma_supp}

\begin{proof}
The proof is very similar to the one found in \cite{papoulis1965probability} but with $n$ variables and by taking gradients with respect to consecutive covariances $\Sigma(i,j)$. For ease of notation, we will refer to $\prod_i^n dx_i$ as $d^n \mathbf{x}$ and that $\sqrt{j} = -1$. First, we define the characteristic function and the inverse Fourier Transform of the joint Gaussian, \ie $f(\mathbf{x})$, as follows:

\begin{equation}
\begin{aligned}
\label{supp:fourier_transform}
\Phi(\mathbf{w}) &= \int_{\mathbb{R}^n}   f(\mathbf{x}) e^{j\prod_{i=1}^n w_i x_i} d^n\mathbf{x}\\
&= e^{\left(\mathbf{w}^\top \mu_x \right)j-\frac{1}{2}\sum_i^n \sum_j^n w_i w_j \Sigma_x(i,j)},
\end{aligned}
\end{equation}

\begin{equation}
\begin{aligned}
\label{supp:inverse_fourier_transform}
f(\mathbf{x}) &= \frac{1}{(2 \pi)^n} \int_{\mathbb{R}^n} \Phi(\mathbf{w})e^{-j \prod_{i=1}^n w_i x_i} d^n\mathbf{w},
\end{aligned}
\end{equation}

\begin{equation}
\begin{aligned}
&\frac{\partial^p f(\mathbf{x})}{\partial^p \mathbf{x}} = \frac{1}{(2 \pi)^n} \int_{\mathbb{R}^n} \Phi(\mathbf{w}) 
 (-j)^p e^{-j \prod_{i=1}^n w_i x_i}\prod_i^p w_i d^n\mathbf{w}.
\end{aligned}
\end{equation}

\noindent Taking the derivatives of the covariances of the consecutive variables, \ie $\rho_{i,i+1}$, we have that
\begin{equation}
\begin{aligned}
\label{supp:derivative_phi}
\frac{\partial^{\frac{p}{2}} \Phi(\mathbf{w})}{\prod_{\forall \text{odd}i} \partial \rho_{i,i+1}} &= \prod_{\forall \text{odd}i}^{\frac{n}{2}}-\frac{1}{2} (2 w_i w_{i+1})\Phi(\mathbf{w}) \\
&= (-1)^{\frac{p}{2}} \Phi(\mathbf{w}) \prod_i^p w_i.
\end{aligned}
\end{equation}

\noindent By substituting Equation \eqref{supp:inverse_fourier_transform} in the expectation of the function $\Psi(\mathbf{x})$ over the joint probability density function $f(\mathbf{x})$, we have:

\begin{equation}
\begin{aligned}
\mathbb{E} [\Psi(\mathbf{x})] &= \int_{\mathbb{R}^n} \Psi(\mathbf{x}) f(\mathbf{x}) d^n\mathbf{x} \\
&= \frac{1}{(2 \pi)^n} \int_{\mathbb{R}^n} \int_{\mathbb{R}^n}  \Psi(\mathbf{x}) e^{-j \prod_{i=1}^n w_ix_i} \Phi(\mathbf{w}) d^n\mathbf{w} d^n\mathbf{x}. \nonumber
\end{aligned}
\end{equation}

Now by applying the theorem and substituting Equation \eqref{supp:derivative_phi}, we have:

\begin{equation}
\begin{aligned}
\frac{\partial^{\frac{p}{2}} \mathbb{E} [\Psi(\mathbf{x})]}{\prod_{\forall \text{odd} i \partial \rho_{ii+1}}} &=  \frac{1}{(2 \pi)^n} \int_{\mathbb{R}^n} \int_{\mathbb{R}^n}
\frac{\partial^{\frac{p}{2}} \Phi(\mathbf{w})}{\prod_{\forall \text{odd} i \partial \rho_{ii+1}}} e^{-j \prod_{i=1}^n w_i x_i} \\
& \qquad
\Psi(\mathbf{x})  d^n\mathbf{w} d^n\mathbf{x} \\
&=  \frac{1}{(2 \pi)^n} \int_{\mathbb{R}^n} \int_{\mathbb{R}^n} (-1)^{\frac{p}{2}} \Phi(\mathbf{w})   \prod_i^p w_i\\
&\qquad \qquad   e^{-j \prod_{i=1}^n w_i x_i}
  \Psi(\mathbf{x}) d^n \mathbf{w} d^n\mathbf{x}\\
&\stackrel{(a)}{=}     \int_{\mathbb{R}^n}  \Psi(\mathbf{x}) \frac{\partial^p f(\mathbf{x})}{\partial^p \mathbf{x}}  d^n\mathbf{x} \\
&=      \int_{\mathbb{R}^n}   \frac{\partial^p \Psi(\mathbf{x})}
{\partial^p\mathbf{x}}  f(\mathbf{x}) d^n\mathbf{x} = \mathbb{E} \left[\frac{\partial^p \Psi(\mathbf{x})}
{\partial^p \mathbf{x}} \right]. \nonumber
\end{aligned}
\end{equation}

\noindent Equality (a) holds since $(-1)^{\frac{p}{2}} = (-j)^p $ holds for even $p$. As for the last equality, it follows by integrating by parts since the Gaussian PDF $f(\mathbf{x})$ is in Schwarz class.
\end{proof}

%%%%%%%%%%%%%%%%%%%%%%%%%%%%%%%%%%%%%%%%%%%%%
\section*{Proof of Lemma \ref{lemma3}}
\begin{lemma_supp}
For any bivariate Gaussian random variable $\mathbf{x}=[x_1,x_2]^{\top}$ $\sim$ $\mathcal{N}(\mathbf{0}_2,\Sigma_x)$, the following holds for $T(x_1,x_2) = \max(x_1,0)\max(x_2,0)$:
\begin{equation}
\begin{aligned}
&\mathbb{E}[T(x_1,x_2)] = \\
&\frac{1}{2\pi} \left(\sigma_{12} \sin^{-1}\left(\frac{\sigma_{12}}{\sigma_1 \sigma_2} \right) +  \sigma_1 \sigma_2 \sqrt{1 - \frac{\sigma_{12}^2}{\sigma_1^2 \sigma_2^2}} \right) + \frac{\sigma_{12}}{4}
\notag
\end{aligned}
\end{equation}
where $\sigma_{ij} = \Sigma_x(i,j) ~\forall i\neq j$ and $\sigma^2_i = \Sigma_x(i,i)$.
\end{lemma_supp}

\begin{proof}
Using Lemma (\ref{price_exten}) with $p = 4$ and choosing $\sigma_{12}$ to be the covariances at which  differentiation happens, we have:
\begin{equation}
\label{pde1}
\begin{aligned}
&\frac{\partial^{2} \mathbb{E}[T(x_1,x_2)]}{\partial \sigma_{12} \partial \sigma_{12}} = \mathbb{E} \left[\frac{\partial^4 \mathbb{E}[T(x_1,x_2)]}{\partial x_1 \partial x_2 \partial x_1 \partial x_2} \right] \\
& = \mathbb{E} \left[\frac{\partial^4 \mathbb{E}[T(x_1,x_2)]}{\partial x_1^2 \partial x_2^2} \right]  = \frac{1}{2 \pi \sigma_1 \sigma_2 \sqrt{1 - \frac{\sigma_{12}^2}{\sigma_1 \sigma_2}}}
\end{aligned}
\end{equation}
To solve the partial differential equation in Equation \eqref{pde1}, two boundary conditions are needed. Similar to \cite{price1958useful}, they can be computed when $\sigma_{12} = 0$, which occurs when $x_1$ and $x_2$ are independent random variables. It is easy to show from Remark \ref{theo1} that $\left. \mathbb{E} [T(x_1,x_2)] \right\vert_{\sigma_{12} = 0} =  \frac{\sigma_1 \sigma_2}{2 \pi}$ and that,
\begin{equation}
    \begin{aligned}
    \left.{\frac{\partial \mathbb{E}[T(x_1,x_2)]}{\partial \sigma_{12}}}\right\vert_{\sigma_{12} = 0} &\stackrel{\text{lemma (\ref{price_exten})}}{=} \mathbb{E}[u(x_1)u(x_2)]  \\
    &\stackrel{{\sigma_{12} = 0}}{=} \mathbb{E}[u(x_1)] \mathbb{E}[u(x_2)] = \frac{1}{4}
    \notag
    \end{aligned}
\end{equation}

\noindent With these boundary conditions, we compute the integral  to complete the proof. 
\begin{equation}
    \begin{aligned}
    \mathbb{E}[T(x_1,x_2)] &= \int_{0}^{\sigma_{12}} \left[\frac{1}{4} + \int_0^y \frac{dc}{2\pi \sqrt{1-\frac{c^2}{\sigma_1^2 \sigma_2^2}}} \right] dy  + \frac{\sigma_1 \sigma_2}{2\pi} \notag
    \end{aligned}
\end{equation}
\end{proof}

%%%%%%%%%%%%%%%%%%%%%%%%%%%%%%%%%%%%%%%%%%%%%
\section*{Proof of Theorem \ref{theo2}}
\begin{proof}
The proof follows naturally after considering the much simpler scalar function that is in the form $\tilde{g}(\mathbf{z}) = \sum_t^d \alpha_t \text{max}(z_t,0)$ where $\mathbf{z}\in \mathbb{R}^d$ $\sim \mathcal{N} \left(\mathbf{0}_d,\Sigma_z \right)$. Therefore, we have $\mathbb{E}[\tilde{g}^2(\mathbf{z})] = \sum_{r}^d \alpha_{r}^2 \mathbb{E}[\text{max}^2\left(z_{r},0\right)] + 2\sum_{v_2 \leq v_1} \alpha_{v_1} \alpha_{v_2}\mathbb{E}[\text{max}\left(z_{v_1},0\right) \text{max}\left(z_{v_2},0\right)]$. Note that $\tilde{g}(\mathbf{z})$ is only a special case of $\mathbf{g}_i(\mathbf{x})$, where $\mathbf{z} = \mathbf{A} \mathbf{x}$ and $\alpha_{t} = \mathbf{B}(i,t)$. It is also clear from $\mathbb{E}[\tilde{g}\left(\mathbf{z}\right)]$ that it is sufficient to analyze $\mathbb{E}[\tilde{g}\left(\mathbf{z}\right)]$ in the bivariate case. Thus, the function we are interested in is  $\mathbb{E}[\tilde{g}^2\left(z_1,z_2 \right)] = \alpha_1^2 \mathbb{E}[\text{max}^2\left(z_1,0\right)] +  \alpha_2^2 \mathbb{E}[\text{max}^2\left(z_2,0\right)] + 2 \alpha_1 \alpha_2 \mathbb{E}[\text{max}\left(z_1,0\right) \text{max}\left(z_2,0\right)]$. Using  Lemmas \ref{lemma1} and \ref{lemma3}, the proof is complete.
\end{proof}

\section*{Recovering Lemma \ref{lemma3} as a Special Case}
\begin{proposition_supp}
Equation \eqref{eq:big_result} for $\mu_1 = \mu_2 = 0$ recovers the result of Lemma \ref{lemma3}.
\end{proposition_supp}
\begin{proof}
To see this, we first start by observing that $\Omega(0,0,\sigma_1,\sigma_2,\rho) = \frac{\sigma_1\sigma_2\sqrt{1-\rho^2}}{2\pi} + \frac{\rho \sigma_1 \sigma_2}{4}$. Note that since $I_{a,b}(x) = \frac{\sqrt{\pi}}{2} \int_{0}^{x} \exp(-t^2)\text{erf}(at+b)dt$, then $I_{a,b}(0) = 0$. Therefore, we have the following:
\begin{equation}
\begin{aligned}
\label{simplified}
&\mathbb{E}[\max({x_1},{0}) \max({x_2},{0})] =  \frac{\sigma_1\sigma_2 \sqrt{1-\rho^2}}{2 \pi} + \frac{\rho \sigma_1 \sigma_2}{4}  \\
& 
\quad \quad  +
\begin{cases}
\frac{\rho\sigma_1\sigma_2}{\pi} I_{\nicefrac{\rho}{\sqrt{1-\rho^2}},0}\left(\infty\right)   ~~~~~~~~~~~~~~~~~~~~~~~~~~~ |\rho|<\frac{1}{\sqrt{2}},\\

\frac{\rho\sigma_1\sigma_2}{\pi} \text{sign}(\rho) \left( \frac{\pi}{4}  -  
I_{\nicefrac{\sqrt{1-\rho^2}{|\rho|}},0}\left(\infty\right) \right) ~~  |\rho|>\frac{1}{\sqrt{2}}.
\end{cases}
\end{aligned}
\end{equation}
Moreover, since for the incomplete Gamma function we have $P(x,\infty) = 1$, then:
\begin{equation}
\begin{aligned}
\label{integral_simplify}
I_{a,0}(\infty) =  \frac{\sqrt{\pi}}{2} \sum_{u=0}^\infty \frac{\left(\nicefrac{a}{2}\right)^{2u+1}}{\Gamma(u + \nicefrac{3}{2})}H_{2u}(0).
\end{aligned}
\end{equation}
This is since the odd terms in the Hermite polynomial vanish, \ie $H_{2u+1}(0) = 0$. As for the even terms, \ie $H_{2u}(x)$, they are given as $H_{2u}(x) = (2u)! \sum_{l=0}^{u}\frac{(-1)^{u-l}}{(2l)! (u - l)!}(2x)^{2l}$. Thus, we have that $H_{2u}(0) = (2u)! \frac{(-1)^u}{u!}$. Note that since, $\Gamma(u + \nicefrac{3}{2}) = (u+\nicefrac{1}{2})\frac{(2u)!}{4^u u!}\sqrt{\pi}$, substituting in \eqref{integral_simplify}, we have the following:
\begin{equation}
\begin{aligned}
I_{a,0}(\infty) =  \frac{1}{2} \sum_{u=0}^\infty \frac{(-1)^u\left(a\right)^{2u+1}}{2u + 1} = \frac{1}{2}\tan^{-1}(a).
\end{aligned}
\end{equation}
Using the identity $\sin^{-1}(x) = \tan^{-1}\left(\nicefrac{x}{{\sqrt{1-x^2}}}\right)$, Equation \eqref{simplified} can be reduced to:
\begin{equation}
\begin{aligned}
&\mathbb{E}[\max({x_1},{0}) \max({x_2},{0})] =  \\
& 
\begin{cases}
\frac{\sigma_1 \sigma_2}{2\pi}\left(\rho \sin^{-1}(\rho) + \sqrt{1-\rho^2}\right)+ \frac{\rho \sigma_1 \sigma_2}{4} 
\quad ~~~ |\rho|<\frac{1}{\sqrt{2}},\\
\frac{\sigma_1 \sigma_2}{2\pi}\left(\rho \text{sign}(\rho) \left[\frac{\pi}{2} - \tan^{-1}\left(\nicefrac{\sqrt{1-\rho^2}}{|\rho|}\right)\right] \right. \\
\left. \qquad \qquad \qquad + \sqrt{1-\rho^2}\right)+ \frac{\rho \sigma_1 \sigma_2}{4} 
\qquad \quad ~~  |\rho|>\frac{1}{\sqrt{2}}.
\end{cases}
\end{aligned}
\end{equation}
Note that for the case $\rho > \nicefrac{1}{\sqrt{2}}$, one can observe that $\nicefrac{\pi}{2} - \tan^{-1}(\frac{\sqrt{1-\rho^2}}{\rho}) = \nicefrac{\pi}{2} - \cos^{-1}(\rho) =  \sin^{-1}(\rho)$. Lastly, when $\rho < -\nicefrac{1}{\sqrt{2}}$, we have that $\nicefrac{\pi}{2} - \tan^{-1}\left(\frac{\sqrt{1-\rho^2}}{-\rho}\right) = \nicefrac{\pi}{2} - \cos^{-1}(-\rho) = - \sin^{-1}(\rho)$ where $\text{sign}(\rho) = -1$. This completes the proof as $\sigma_{12} = \rho \sigma_1 \sigma_2$ in Lemma \ref{lemma3}.
\end{proof}

% use section* for acknowledgment
\ifCLASSOPTIONcompsoc
  % The Computer Society usually uses the plural form
  \section*{Acknowledgments}
\else
  % regular IEEE prefers the singular form
  \section*{Acknowledgment}
\fi
% \vspace{-0.15cm}
% King Abdullah University of Science and Technology (KAUST) Office of Sponsored Research.
This work was supported by King Abdullah University of Science and Technology (KAUST) Office of Sponsored Research.

% Can use something like this to put references on a page
% by themselves when using endfloat and the captionsoff option.
\ifCLASSOPTIONcaptionsoff
  \newpage
\fi

% trigger a \newpage just before the given reference
% number - used to balance the columns on the last page
% adjust value as needed - may need to be readjusted if
% the document is modified later
%\IEEEtriggeratref{8}
% The "triggered" command can be changed if desired:
%\IEEEtriggercmd{\enlargethispage{-5in}}

% references section
\bibliographystyle{IEEEtran}
\bibliography{references.bib}

% Generated by IEEEtran.bst, version: 1.14 (2015/08/26)
\begin{thebibliography}{10}
\providecommand{\url}[1]{#1}
\csname url@samestyle\endcsname
\providecommand{\newblock}{\relax}
\providecommand{\bibinfo}[2]{#2}
\providecommand{\BIBentrySTDinterwordspacing}{\spaceskip=0pt\relax}
\providecommand{\BIBentryALTinterwordstretchfactor}{4}
\providecommand{\BIBentryALTinterwordspacing}{\spaceskip=\fontdimen2\font plus
\BIBentryALTinterwordstretchfactor\fontdimen3\font minus
  \fontdimen4\font\relax}
\providecommand{\BIBforeignlanguage}[2]{{%
\expandafter\ifx\csname l@#1\endcsname\relax
\typeout{** WARNING: IEEEtran.bst: No hyphenation pattern has been}%
\typeout{** loaded for the language `#1'. Using the pattern for}%
\typeout{** the default language instead.}%
\else
\language=\csname l@#1\endcsname
\fi
#2}}
\providecommand{\BIBdecl}{\relax}
\BIBdecl

\bibitem{bibi2018analytic}
A.~Bibi, M.~Alfadly, and B.~Ghanem, ``Analytic expressions for probabilistic
  moments of pl-dnn with gaussian input,'' in \emph{IEEE Conference on Computer
  Vision and Pattern Recognition (CVPR)}, 2018.

\bibitem{lecun2015yoshua}
Y.~LeCun, Y.~Bengio, and G.~Hinton, ``Deep learning,'' \emph{Nature}, 2015.

\bibitem{krizhevsky2012imagenet}
A.~Krizhevsky, I.~Sutskever, and G.~E. Hinton, ``Imagenet classification with
  deep convolutional neural networks,'' in \emph{Advances in Neural Information
  Processing Systems (NeurIPS)}, 2012.

\bibitem{hinton2012deep}
G.~Hinton, L.~Deng, D.~Yu, G.~E. Dahl, A.-r. Mohamed, N.~Jaitly, A.~Senior,
  V.~Vanhoucke, P.~Nguyen, T.~N. Sainath \emph{et~al.}, ``Deep neural networks
  for acoustic modeling in speech recognition: The shared views of four
  research groups,'' \emph{IEEE Signal Processing Magazine}, 2012.

\bibitem{szegedy2013intriguing}
C.~Szegedy, W.~Zaremba, I.~Sutskever, J.~Bruna, D.~Erhan, I.~Goodfellow, and
  R.~Fergus, ``Intriguing properties of neural networks,'' in
  \emph{International Conference on Learning Representations (ICLR)}, 2014.

\bibitem{goodfellow2014explaining}
I.~J. Goodfellow, J.~Shlens, and C.~Szegedy, ``Explaining and harnessing
  adversarial examples,'' \emph{arXiv:1412.6572}, 2014.

\bibitem{moosavi2016deepfool}
S.-M. Moosavi-Dezfooli, A.~Fawzi, and P.~Frossard, ``Deepfool: a simple and
  accurate method to fool deep neural networks,'' in \emph{IEEE Conference on
  Computer Vision and Pattern Recognition (CVPR)}, 2016.

\bibitem{moosavi2016universal}
S.-M. Moosavi-Dezfooli, A.~Fawzi, O.~Fawzi, and P.~Frossard, ``Universal
  adversarial perturbations,'' in \emph{IEEE Conference on Computer Vision and
  Pattern Recognition (CVPR)}, 2017.

\bibitem{alfadly2019train}
M.~Alfadly, A.~Bibi, and B.~Ghanem, ``Analytical moment regularizer for
  gaussian robust networks,'' in \emph{arXiv:1904.11005}, 2019.

\bibitem{lecun1998mnist}
Y.~LeCun, ``The mnist database of handwritten digits,'' \emph{http://yann.
  lecun. com/exdb/mnist/}, 1998.

\bibitem{goodfellow2015}
I.~J. Goodfellow, D.~Erhan, P.~L. Carrier, A.~Courville, M.~Mirza, B.~Hamner,
  W.~Cukierski, Y.~Tang, D.~Thaler, D.-H. Lee, Y.~Zhou, C.~Ramaiah, F.~Feng,
  R.~Li, X.~Wang, D.~Athanasakis, J.~Shawe-Taylor, M.~Milakov, J.~Park,
  R.~Ionescu, M.~Popescu, C.~Grozea, J.~Bergstra, J.~Xie, L.~Romaszko, B.~Xu,
  Z.~Chuang, and Y.~Bengio, ``Challenges in representation learning: A report
  on three machine learning contests,'' \emph{Neural Networks}, 2015.

\bibitem{su2017one}
J.~Su, D.~V. Vargas, and S.~Kouichi, ``One pixel attack for fooling deep neural
  networks,'' \emph{IEEE Transactions on Evolutionary Computation}, 2019.

\bibitem{fawzi2016measuring}
A.~Fawzi and P.~Frossard, ``Measuring the effect of nuisance variables on
  classifiers,'' in \emph{British Machine Vision Conference (BMVC)}, 2016.

\bibitem{fawzi2015manitest}
------, ``Manitest: Are classifiers really invariant?'' in \emph{British
  Machine Vision Conference (BMVC)}, 2015.

\bibitem{madry2017towards}
A.~Madry, A.~Makelov, L.~Schmidt, D.~Tsipras, and A.~Vladu, ``Towards deep
  learning models resistant to adversarial attacks,'' \emph{International
  Conference on Learning Representations (ICLR)}, 2018.

\bibitem{cheng2017maximum}
C.-H. Cheng, G.~N{\"u}hrenberg, and H.~Ruess, ``Maximum resilience of
  artificial neural networks,'' in \emph{International Symposium on Automated
  Technology for Verification and Analysis}, 2017.

\bibitem{lomuscio2017approach}
A.~Lomuscio and L.~Maganti, ``An approach to reachability analysis for
  feed-forward relu neural networks,'' \emph{arXiv:1706.07351}, 2017.

\bibitem{scheibler2015towards}
K.~Scheibler, L.~Winterer, R.~Wimmer, and B.~Becker, ``Towards verification of
  artificial neural networks,'' in \emph{MBMV}, 2015.

\bibitem{katz2017reluplex}
G.~Katz, C.~Barrett, D.~L. Dill, K.~Julian, and M.~J. Kochenderfer, ``Reluplex:
  An efficient smt solver for verifying deep neural networks,'' in
  \emph{International Conference on Computer Aided Verification}, 2017.

\bibitem{zhang2018efficient}
H.~Zhang, T.-W. Weng, P.-Y. Chen, C.-J. Hsieh, and L.~Daniel, ``Efficient
  neural network robustness certification with general activation functions,''
  in \emph{Advances in Neural Information Processing Systems (NeurIPS)}, 2018.

\bibitem{wong2017provable}
E.~Wong and J.~Z. Kolter, ``Provable defenses against adversarial examples via
  the convex outer adversarial polytope,'' \emph{International Conference on
  Machine Learning (ICML)}, 2018.

\bibitem{lecuyer2019certified}
M.~Lecuyer, V.~Atlidakis, R.~Geambasu, D.~Hsu, and S.~Jana, ``Certified
  robustness to adversarial examples with differential privacy,'' in \emph{IEEE
  Symposium on Security and Privacy (SP)}, 2019.

\bibitem{cohen_randomized_1}
J.~Cohen, E.~Rosenfeld, and Z.~Kolter, ``Certified adversarial robustness via
  randomized smoothing,'' in \emph{International Conference on Machine Learning
  (ICML)}, 2019.

\bibitem{gast2018lightweight}
J.~Gast and S.~Roth, ``Lightweight probabilistic deep networks,'' in \emph{IEEE
  Conference on Computer Vision and Pattern Recognition (CVPR)}, 2018.

\bibitem{price1958useful}
R.~Price, ``A useful theorem for nonlinear devices having gaussian inputs,''
  \emph{IEEE Transactions on Information Theory}, 1958.

\bibitem{mcmahon1964extension}
E.~McMahon, ``An extension of price's theorem (corresp.),'' \emph{IEEE
  Transactions on Information Theory}, 1964.

\bibitem{fayed2014evaluation}
H.~Fayed and A.~Atiya, ``An evaluation of the integral of the product of the
  error function and the normal probability density with application to the
  bivariate normal integral,'' \emph{Mathematics of Computation}, 2014.

\bibitem{modas019sparsefool}
A.~Modas, S.-M. Moosavi-Dezfooli, and P.~Frossard, ``Sparsefool: a few pixels
  make a big difference,'' in \emph{IEEE Conference on Computer Vision and
  Pattern Recognition (CVPR)}, 2019.

\bibitem{lecun1999object}
Y.~LeCun, P.~Haffner, L.~Bottou, and Y.~Bengio, ``Object recognition with
  gradient-based learning,'' \emph{Shape, contour and grouping in computer
  vision}, 1999.

\bibitem{papoulis1965probability}
A.~Papoulis, \emph{Probability, Random Variables, and Stochastic Processes},
  ser. International student edition.\hskip 1em plus 0.5em minus 0.4em\relax
  McGraw-Hill, 1965.

\end{thebibliography}

% biography section
\vspace{-1.20cm}
\begin{IEEEbiography}[{\includegraphics[width=1in,height=1.25in,clip,keepaspectratio]{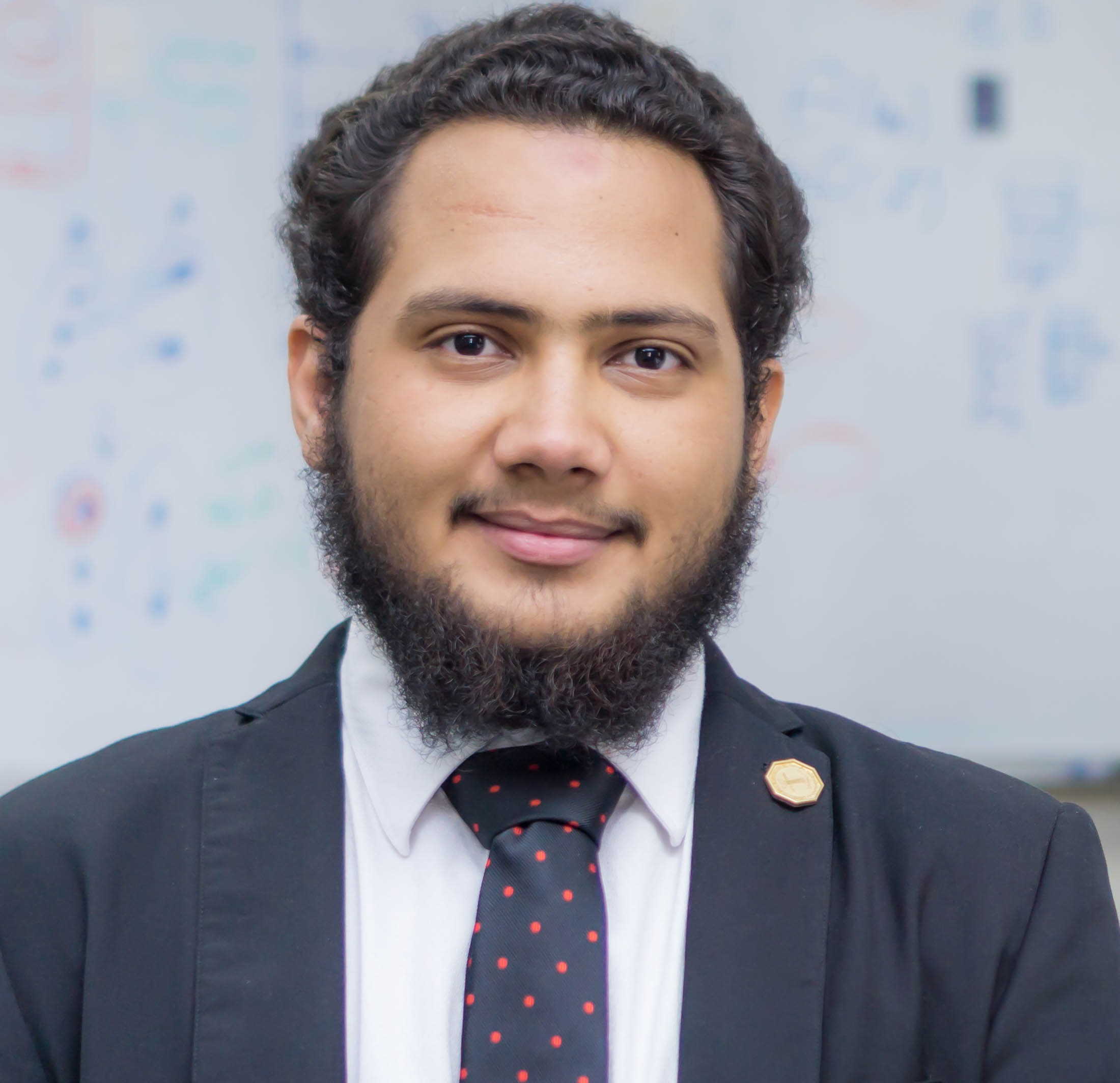}}]{Modar Alfadly} is pursuing his PhD degree in computer science at King Abdullah University of Science and Technology (KAUST) where he also obtained his Masters degree in 2018. In 2016, he received his BSc degree with honors distinction in software engineering from King Fahd University of Petroleum and Minerals (KFUPM). Currently, his main focus is understanding deep neural networks for computer vision. His recent research interests include network robustness, adversarial attacks, and uncertainty prediction.
\end{IEEEbiography}

\vspace{-1.20cm}
\begin{IEEEbiography}[{\includegraphics[width=1in,height=1.25in,clip,keepaspectratio]{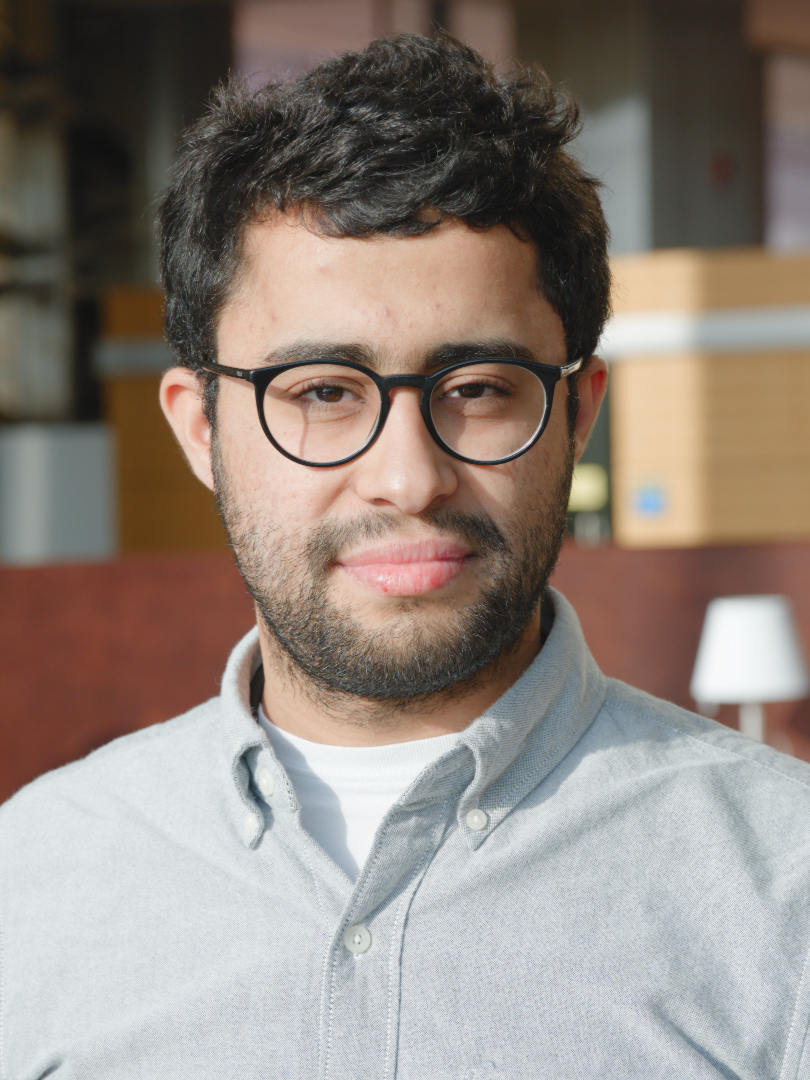}}]{Adel Bibi} received his BSc degree in electrical engineering with class honors from Kuwait university in 2014. He later obtained his MSc degree in electrical engineering with a focus on computer vision from King Abdullah University of Science and Technology (KAUST) in 2016. Currently, he is pursuing his PhD degree with a focus at the intersection between computer vision, machine learning and optimization in KAUST. Adel has been recognized as an outstanding reviewer for CVPR18, CVPR19 and ICCV19 and won the best paper award at the optimization and big data conference in KAUST. He has published more than 10 papers in CVPRs, ECCVs, ICCVs and ICLRs some which were selected for oral and spotlight presentations.
\end{IEEEbiography}

\vspace{-1.20cm}
\begin{IEEEbiography}[{\includegraphics[width=1in,height=1.25in,clip,keepaspectratio]{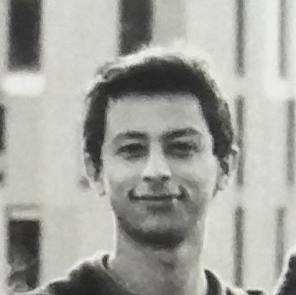}}]{Emilio Botero} obtained his BSc degree in biomedical engineering in 2017 from the Universidad de los Andes in Bogotá, Colombia. After graduating, he did an internship at KAUST, where he worked on understanding neural networks. He is currently finishing his MSc in Computer Science from Université de Montréal/Mila (Quebec Artificial Intelligence Institute), with emphasis in machine learning, representation learning and probabilistic graphical models.
\end{IEEEbiography}

\vspace{-1.20cm}
\begin{IEEEbiography}[{\includegraphics[width=1in,height=1.25in,clip,keepaspectratio]{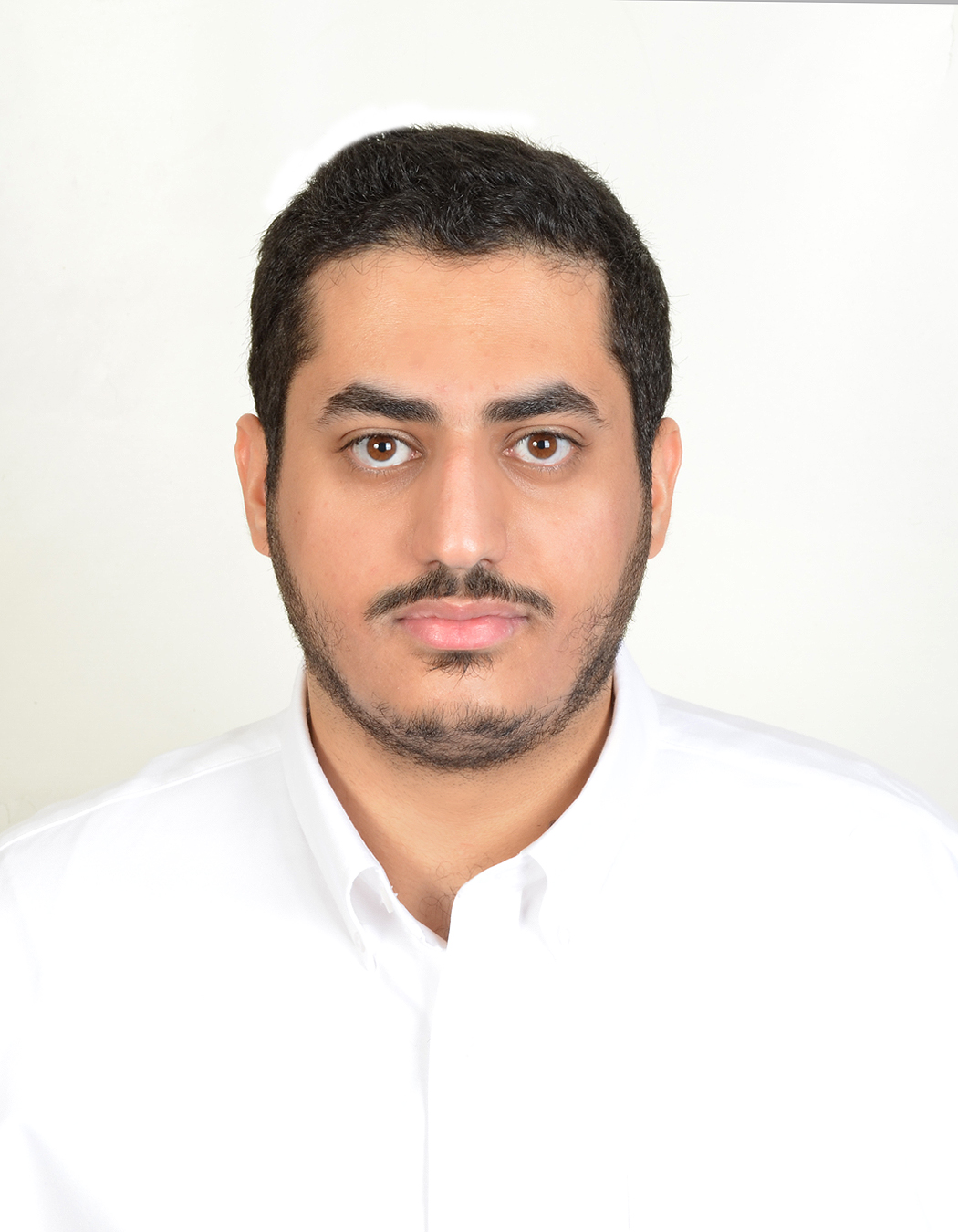}}]{Salman Alsubaihi} received his BSc degree in Electrical Engineering from King Fahd University of Petroleum and Minerals (KFUPM) in 2017. He then obtained his MSc degree in Electrical Engineering, with focus on computer vision, from King Abdullah University of Science and Technology (KAUST) in 2019. Salman's research interests are toward the study and analysis of neural networks.
\end{IEEEbiography}

\vspace{-1.20cm}
\begin{IEEEbiography}[{\includegraphics[width=1in,height=1.25in,clip,keepaspectratio]{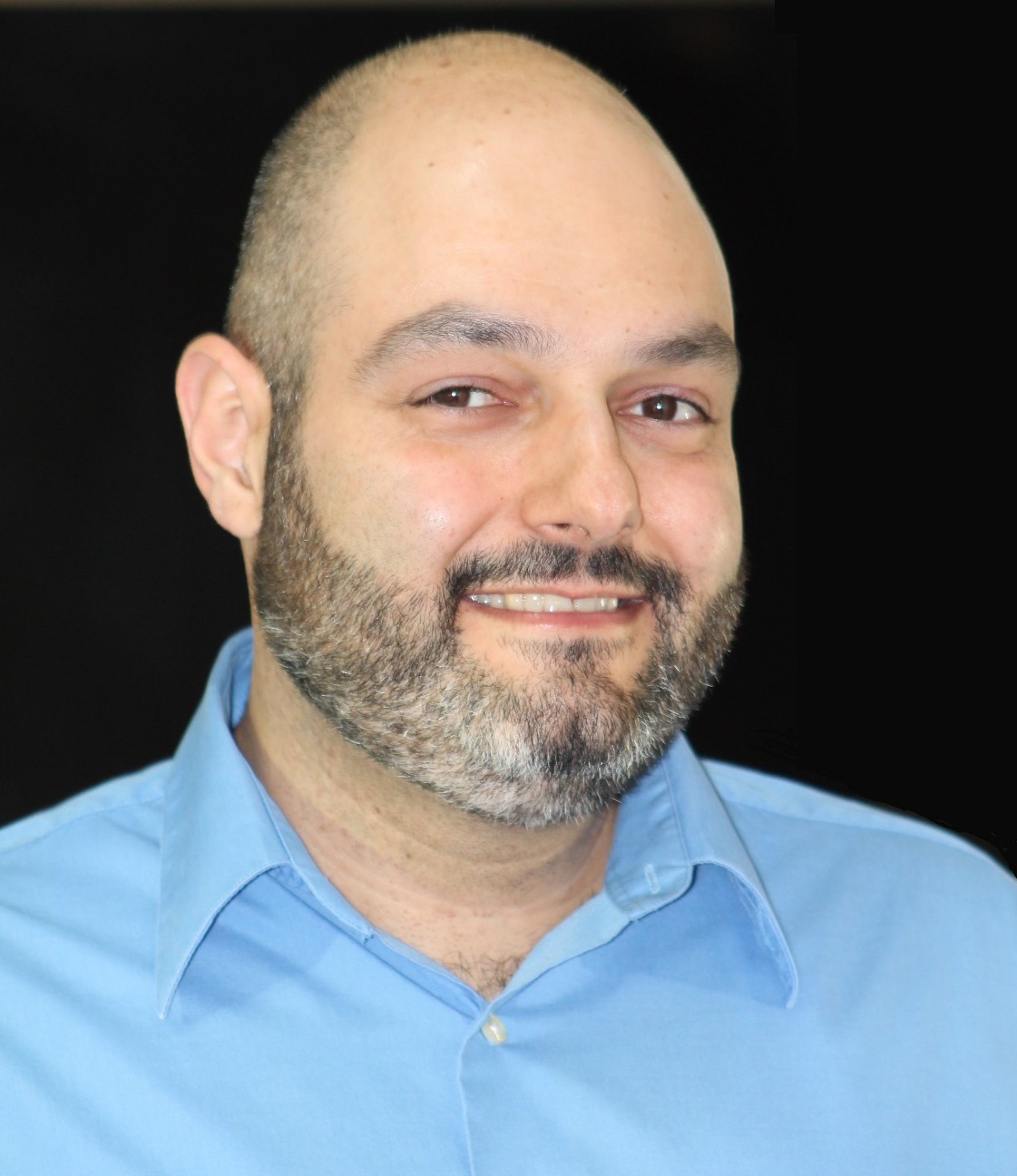}}]{Bernard Ghanem} is currently an associate professor with the King Abdullah University of Science and Technology (KAUST), in the Visual Computing Center (VCC). He leads the Image and Video Understanding Lab (IVUL), KAUST. He is involved in several interesting projects that focus on exploiting techniques in computer vision and machine learning for real-world applications including semantic sports video analysis, large-scale activity recognition/detection, and real-time crowd analysis. He has published more than 100 peer-papers in peer-reviewed venues including the IEEE Transactions on Pattern Analysis and Machine Intelligence, the International Journal of Computer Vision, CVPR, ICCV, ECCV, etc. He is a member of the IEEE.
\end{IEEEbiography}

% \enlargethispage{-5in}

\end{document}